\documentclass[10pt]{article}

\let\stdauthor\author
\usepackage{bofflab}

\usepackage[utf8]{inputenc}
\usepackage[T1]{fontenc}
\usepackage{fullpage}
\usepackage{microtype}
\usepackage{parskip}
\usepackage{ragged2e}

\usepackage{amsmath}
\usepackage{amsfonts}
\usepackage{amssymb}
\usepackage{amsthm}
\usepackage{mathrsfs}
\usepackage{mathtools}
\usepackage{nicefrac}
\usepackage{thmtools}
\usepackage{thm-restate}

\usepackage{graphicx}
\usepackage{booktabs}
\usepackage{colortbl}
\usepackage{multirow}
\usepackage{placeins}
\usepackage{wrapfig}
\usepackage{float}
\usepackage[font=small,labelfont=bf]{caption}
\usepackage{subcaption}
\usepackage{enumerate}
\usepackage{enumitem}
\usepackage{etoolbox}
\usepackage[ruled,vlined]{algorithm2e}
\usepackage{tikz}
\usetikzlibrary{arrows.meta, fit, positioning, calc}
\usepackage{pgfplots}
\pgfplotsset{compat=1.18}

\usepackage[capitalise,noabbrev]{cleveref}

\crefformat{equation}{(#2#1#3)}
\crefrangeformat{equation}{(#3#1#4) to (#5#2#6)}
\crefmultiformat{equation}{(#2#1#3)}{ and (#2#1#3)}{, (#2#1#3)}{ and (#2#1#3)}

\crefname{appendix}{Appendix}{Appendices}
\Crefname{appendix}{Appendix}{Appendices}

\crefname{assumption}{Assumption}{Assumptions}
\Crefname{assumption}{Assumption}{Assumptions}

\makeatletter
\let\c@table\c@figure
\@ifundefined{c@algocf}{}{\let\c@algocf\c@figure}
\makeatother

\DeclareMathOperator*{\argmax}{argmax}

\newcommand{\R}{\ensuremath{\mathbb{R}}}

\newcommand{\E}{\mathbb{E}}

\newcommand{\calL}{\mathcal{L}}

\newcommand{\calR}{\mathcal{R}}

\newcommand{\calP}{\mathcal{P}}
\newcommand{\calU}{\mathcal{U}}
\newcommand{\calW}{\mathcal{W}}

\newcommand{\norm}[1]{\lVert #1 \rVert}

\newcommand{\Normal}{\mathcal{N}}
\newcommand{\Unif}{\mathrm{Unif}}

\newcommand{\Law}{\mathrm{Law}}
\newcommand{\supp}{\mathrm{supp}}

\newcommand{\kl}[2]{\mathrm{KL}\left(#1 \| #2\right)}

\newcommand{\Id}{I}

\newcommand{\sg}[1]{\mathrm{sg}\left(#1\right)}

\newcommand{\lsd}{\mathrm{LSD}}
\newcommand{\esd}{\mathrm{ESD}}
\newcommand{\psd}{\mathrm{PSD}}
\newcommand{\dist}{\mathrm{dist}}

\newcommand{\Tprior}{\mathcal{T}_b}             %
\newcommand{\cprior}{c_b}                       %
\newcommand{\Lprior}{\mathcal{L}_b}                       %
\newcommand{\WTF}{\mathrm{WTF}}                 %
\newcommand{\SBeps}{\mathrm{SB}_\varepsilon}    %

\newcommand{\p}{\partial}

\usepackage{xspace}

\definecolor{ourrowcolor}{RGB}{232,228,245}
\newcommand{\ourrow}{\rowcolor{ourrowcolor}}

\definecolor{keyblue}{RGB}{229,238,249}
\definecolor{keyblueborder}{RGB}{41,112,204}

\newtcolorbox{keyresult}{
	enhanced,
	colback=keyblue,
	colframe=keyblueborder,
	boxrule=1pt,
	arc=3pt,
	left=8pt,
	right=8pt,
	top=6pt,
	bottom=6pt,
	breakable
}

\newtcolorbox{highlight}{
	enhanced,
	colback=keyblue,
	colframe=keyblue,
	boxrule=0pt,
	arc=2pt,
	left=8pt,
	right=8pt,
	top=6pt,
	bottom=6pt,
	breakable
}

\definecolor{keygray}{RGB}{245,245,245}
\definecolor{keygrayborder}{RGB}{200,200,200}

\newtcolorbox{graybox}{
	enhanced,
	colback=keygray,
	colframe=keygrayborder,
	boxrule=0.5pt,
	arc=3pt,
	left=8pt,
	right=8pt,
	top=6pt,
	bottom=6pt,
	breakable
}

\makeatletter
\@ifundefined{newcounteralias}{}{%
  \renewcommand\thmt@autorefsetup{\@xa\def\csname\thmt@envname autorefname\@xa\endcsname\@xa{\thmt@thmname}}%
}
\makeatother

\theoremstyle{plain}
\declaretheorem[name=Theorem, numberwithin=section]{theorem}
\declaretheorem[name=Proposition, sibling=theorem]{proposition}
\declaretheorem[name=Lemma, sibling=theorem]{lemma}

\theoremstyle{definition}

\declaretheorem[name=Assumption, sibling=theorem, style=definition]{assumption}

\theoremstyle{remark}

\title{WTF?!\\Simulation-Free Reinforcement Learning with Wasserstein-Tilted Flow Maps}

\author[1,\dagger,*]{Abbas Mammadov}
\author[2,\dagger,*]{Jerry Y. Huang}
\author[2,*]{Justin Lin}
\author[2]{\protect \\ Partha Kaushik}
\author[2]{Sheel Shah}
\author[2]{Kartik Nair}
\author[1]{\protect \\ Yee Whye Teh}
\author[2]{Nicholas M. Boffi}

\affiliation[1]{University of Oxford}
\affiliation[2]{Carnegie Mellon University}
\contribution[]{\textbf{Code:} \url{https://github.com/wasserstein-tilted-flow-maps/WTF}}

\stdauthor{%
  Abbas Mammadov\textsuperscript{1,\textdagger,*}%
    \thanks{\textsuperscript{\textdagger}Project lead.
      \textsuperscript{*}Equal contribution.
      \textsuperscript{1}University of Oxford,
      \textsuperscript{2}Carnegie Mellon University.
      Correspondence to \texttt{abbas.mammadov@stats.ox.ac.uk} and
      \texttt{jerryhua@andrew.cmu.edu}.}
  \and Jerry Y. Huang\textsuperscript{2,\textdagger,*}
  \and Justin Lin\textsuperscript{2,*}
  \and Partha Kaushik\textsuperscript{2}
  \and Sheel Shah\textsuperscript{2}
  \and Kartik Nair\textsuperscript{2}
  \and Yee Whye Teh\textsuperscript{1}
  \and Nicholas M. Boffi\textsuperscript{2}}

\newcommand{\blfootnote}[1]{%
	\begingroup
	\renewcommand{\thefootnote}{}\footnote{#1}%
	\addtocounter{footnote}{-1}%
	\endgroup
}

\begin{document}

\begin{abstract}
	\emph{Reward fine-tuning} aims to update a pre-trained flow-based generative model to improve the downstream reward of its generated samples.
	Existing methods typically formulate this problem as sampling from a reward-tilted distribution, the solution to a KL-regularized reward-maximization problem.
	Here, we introduce an \textit{optimal transport} regularizer built directly from the pre-trained drift.
	Unlike KL reward tilting, the resulting objective transports individual samples toward higher reward rather than reweighting the base distribution.
	We show that the resulting problem is equivalent to a \textit{deterministic optimal control problem} on the flow.
	Given a pre-trained flow map, this equivalence yields a simulation-free reinforcement learning algorithm for fine-tuning generative flows.
	We call the resulting framework \textbf{Wasserstein-Tilted Flow Maps (WTF)}, the first end-to-end fine-tuning recipe native to flow maps.
  The output is a fine-tuned flow map that retains strong reward-aligned performance at few-step inference budgets without post-hoc distillation.
  Experiments on ImageNet-256 and text-to-image show that WTF achieves higher reward with comparable or higher diversity than baselines, while requiring up to \textbf{$280\times$} less training compute.
  More broadly, we argue that accelerated samplers such as flow maps are essential infrastructure for efficient post-training, and that the dominant KL-regularized formulation is only one of many choices worth revisiting.
\end{abstract}

\null\vspace{-0.25in}\vspace{-\baselineskip}\vspace{-\parskip} 
\maketitle
\enlargethispage{0.5in}
\blfootnote{$^{\dagger}$Project lead. $^{*}$Equal contribution. Correspondence to \texttt{abbas.mammadov@stats.ox.ac.uk} and \texttt{jerryhua@andrew.cmu.edu}.}

\vspace{-0.35in}
\begin{figure}[H]
	\centering
	\includegraphics[width=0.96\linewidth]{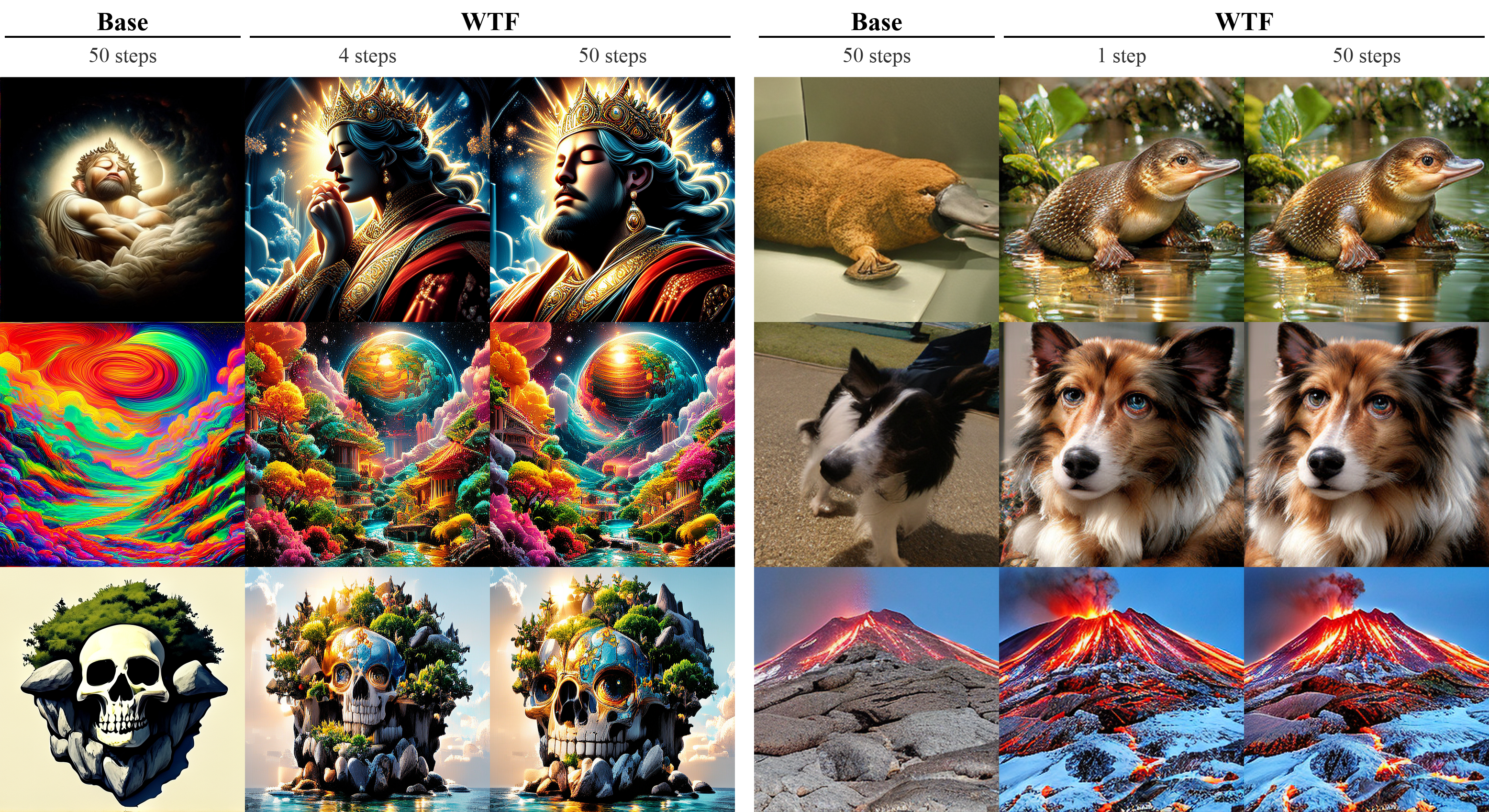}
	\caption{\textbf{Reward-aligned generation at few steps.}
		Flow maps fine-tuned with WTF for HPSv2~\citep{wu2023humanpreferencescorev2}, a learned
		human-preference reward.
		WTF at $4$ steps (text-to-image, left) and $1$ step (ImageNet-256, right) matches
		the quality of its own $50$-step generation; the base model is shown at $50$ steps.
		Prompts are given in~\cref{app:qual_t2i}.}
	\label{fig:front_hero}
\end{figure}

\section{Introduction}
\label{sec:intro}
Flow-based generative models are state-of-the-art for high-fidelity synthesis across continuous and discrete modalities including images~\citep{rombach2022high, esser_scaling_2024}, video~\citep{blattmann2023stable}, protein structure~\citep{abramson2024}, materials~\citep{zeni2025}, and text~\citep{lee2026flowmaplanguagemodels, roos2026categoricalflowmaps, potaptchik2026discreteflowmaps}.
These models are pre-trained on large corpora of data to enable generation from the distribution observed during training, but in practice, we rarely desire samples without further preference.
Applications in creative generation, language modeling, and scientific design generically require samples that are likely under the data distribution but which also attain a high score under a \textit{reward function} $r$, which quantifies, for example, alignment with human aesthetic preferences~\citep{xu2023imagerewardlearningevaluatinghuman,wu2023humanpreferencescorev2}, adherence to safety guidelines or cultural norms~\citep{ouyang2022training,gao2023scaling}, or suitability for downstream tasks such as binding affinity~\citep{watson2023,corso2023diffdock}.
The goal of \emph{reward fine-tuning}\footnote{We write ``fine-tuning'' for ``reward fine-tuning'' throughout.} is to leverage an additional post-training phase to update the model's weights so that its samples improve $r$ while remaining close to its base distribution.
The supervision comes from the reward, with no target samples provided, making this a reinforcement learning problem.
Most of these methods require expensive rollouts from the pre-trained model to score terminal samples under $r$, which makes rollout cost a major component of the post-training budget.

The prevailing theoretical formulation casts reward fine-tuning as sampling from a \emph{reward-tilted} distribution $\tilde{\rho}_1(x) \propto e^{\lambda\, r(x)}\, \rho_1(x)$, where $\rho_1$ is the model's sampling distribution and $\lambda > 0$ is the inverse temperature reward scale.
This reward tilt is the solution to a KL-regularized reward-maximization problem.
Qualitatively, this definition reweights samples from $\rho_1$ according to the reward, making samples with higher reward more likely even if they were rare under the original flow.
Significant recent effort has gone into algorithms that approximate this distribution~\citep{clark2024directlyfinetuningdiffusionmodels, domingoenrich2025adjointmatchingfinetuningflow, potaptchik2026metaflowmapsenable, holderrieth2025diamond, uehara2025inferencetimealignmentdiffusionmodels} through an equivalent diffusion process, rather than operating on the flow itself.
While seemingly awkward, this occurs because the corresponding pathwise KL regularizer can be computed efficiently for diffusions via Girsanov's theorem~\citep{karatzas2014brownian}.
Existing methods therefore adopt an approach in which the flow is first converted into a diffusion, the diffusion is fine-tuned, and the fine-tuned diffusion is converted back into a flow.

The fine-tuning procedure is therefore chiefly diffusion-based, while the field has primarily moved towards the use of deterministic flows.
Furthermore, increasing attention has been placed on the distillation of flows into few-step \emph{flow maps}, which amortize the inference process and produce a sample in as few as one network evaluation~\citep{boffi_flow_2024, boffi2025self, song2023consistencymodels, kim_consistency_2024, geng2025meanflowsonestepgenerative}.
Recently, Flow Map Reward Guidance (FMRG)~\citep{huang2026guideflowfewstepalignment} suggested that this deterministic structure can be exploited for efficient reward alignment without relying on reward tilting, achieving few-step inference-time alignment through deterministic optimal control.
For reward fine-tuning, however, existing procedures remain reliant on diffusion and cannot exploit these accelerated state transitions during training.
A flow-centric fine-tuning paradigm could leverage these distilled generators for dramatically accelerated rollouts, significantly reducing the central computational burden of post-training.

This mismatch between flow map deployment and existing fine-tuning procedures motivates the central question of our work:
\begin{center}
	\emph{Is there an efficient framework for direct end-to-end fine-tuning of a deterministic flow map?}
\end{center}
\begin{figure}[tb]
\centering
\usetikzlibrary{decorations.markings}

\definecolor{baselavfill}{RGB}{232, 228, 245}
\definecolor{baselavline}{RGB}{155, 150, 180}
\definecolor{ftpurplefill}{RGB}{170, 145, 195}
\definecolor{ftpurpleline}{RGB}{75, 50, 115}
\definecolor{rewardgreen}{RGB}{80, 130, 80}

\begin{tikzpicture}[
    x=1cm, y=1cm,
    every node/.style={font=\small},
]
\pgfresetboundingbox
\useasboundingbox (-3.30, -0.40) rectangle (10.40, 2.10);

\begin{scope}[shift={(0, 0)}, xscale=2, yscale=1]

    \node[anchor=south, font=\small\bfseries, black] at (0, 1.65)
        {Wasserstein tilt (ours)};
    \node[anchor=south, font=\footnotesize, black] at (0, 0.95)
        {$\arg\max_\nu\, \big\{\, \lambda\, \E_\nu[r] - \Tprior(\rho_0, \nu)\, \big\}$};

    \draw[->, thick] (-1.5, 0) -- (1.55, 0) node[anchor=west, font=\footnotesize] {$x$};

    \begin{scope}[yscale=1.4]
        \fill[rewardgreen, opacity=0.12]
            plot[smooth, variable=\x, domain=-1.5:1.5, samples=240]
                (\x, {0.2 + 0.3*exp(-8*(\x-0.7)*(\x-0.7))
                          + 0.4*exp(-6*(\x+0.75)*(\x+0.75))})
            -- (1.5, 0) -- (-1.5, 0) -- cycle;
        \draw[rewardgreen, thick]
            plot[smooth, variable=\x, domain=-1.5:1.5, samples=240]
                (\x, {0.2 + 0.3*exp(-8*(\x-0.7)*(\x-0.7))
                          + 0.4*exp(-6*(\x+0.75)*(\x+0.75))});
    \end{scope}

    \begin{scope}[yscale=2]
        \fill[baselavfill, opacity=0.85]
            plot[smooth, variable=\x, domain=-1.5:1.5, samples=300]
                (\x, {0.25*exp(-36*(\x-1)*(\x-1))   + 0.25*exp(-36*(\x+1)*(\x+1))
                    + 0.25*exp(-18*(\x-0.3)*(\x-0.3)) + 0.25*exp(-18*(\x+0.3)*(\x+0.3))})
            -- (1.5, 0) -- (-1.5, 0) -- cycle;
        \draw[baselavline, line width=0.5pt]
            plot[smooth, variable=\x, domain=-1.5:1.5, samples=300]
                (\x, {0.25*exp(-36*(\x-1)*(\x-1))   + 0.25*exp(-36*(\x+1)*(\x+1))
                    + 0.25*exp(-18*(\x-0.3)*(\x-0.3)) + 0.25*exp(-18*(\x+0.3)*(\x+0.3))});

        \fill[ftpurplefill, opacity=0.55]
            plot[smooth, variable=\x, domain=-1.5:1.5, samples=300]
                (\x, {0.25*exp(-36*(\x-0.83)*(\x-0.83))   + 0.25*exp(-36*(\x+0.80)*(\x+0.80))
                    + 0.25*exp(-18*(\x-0.48)*(\x-0.48))   + 0.25*exp(-18*(\x+0.52)*(\x+0.52))})
            -- (1.5, 0) -- (-1.5, 0) -- cycle;
        \draw[ftpurpleline, line width=0.7pt]
            plot[smooth, variable=\x, domain=-1.5:1.5, samples=300]
                (\x, {0.25*exp(-36*(\x-0.83)*(\x-0.83))   + 0.25*exp(-36*(\x+0.80)*(\x+0.80))
                    + 0.25*exp(-18*(\x-0.48)*(\x-0.48))   + 0.25*exp(-18*(\x+0.52)*(\x+0.52))});

        \draw[ftpurpleline, line width=0.85pt, -{Stealth[length=4pt, width=3pt]}]
            ( 1.00, 0.34) -- ( 0.83, 0.34);
        \draw[ftpurpleline, line width=0.85pt, -{Stealth[length=4pt, width=3pt]}]
            (-1.00, 0.34) -- (-0.80, 0.34);
        \draw[ftpurpleline, line width=0.85pt, -{Stealth[length=4pt, width=3pt]}]
            ( 0.30, 0.34) -- ( 0.48, 0.34);
        \draw[ftpurpleline, line width=0.85pt, -{Stealth[length=4pt, width=3pt]}]
            (-0.30, 0.34) -- (-0.52, 0.34);
    \end{scope}

    \node[anchor=north, font=\scriptsize, ftpurpleline] at (0, -0.05)
        {local transport toward nearby high-reward regions};
\end{scope}

\begin{scope}[shift={(7, 0)}, xscale=2, yscale=1]

    \node[anchor=south, font=\small\bfseries, black] at (0, 1.65)
        {Reward tilt};
    \node[anchor=south, font=\footnotesize, black] at (0, 0.95)
        {$\arg\max_\nu\, \big\{\, \lambda\, \E_\nu[r] - \kl{\nu}{\rho_1}\, \big\} \propto e^{\lambda r} \rho_1$};

    \draw[->, thick] (-1.5, 0) -- (1.55, 0) node[anchor=west, font=\footnotesize] {$x$};

    \begin{scope}[yscale=1.4]
        \fill[rewardgreen, opacity=0.12]
            plot[smooth, variable=\x, domain=-1.5:1.5, samples=240]
                (\x, {0.2 + 0.3*exp(-8*(\x-0.7)*(\x-0.7))
                          + 0.4*exp(-6*(\x+0.75)*(\x+0.75))})
            -- (1.5, 0) -- (-1.5, 0) -- cycle;
        \draw[rewardgreen, thick]
            plot[smooth, variable=\x, domain=-1.5:1.5, samples=240]
                (\x, {0.2 + 0.3*exp(-8*(\x-0.7)*(\x-0.7))
                          + 0.4*exp(-6*(\x+0.75)*(\x+0.75))});
    \end{scope}

    \begin{scope}[yscale=2]
        \fill[baselavfill, opacity=0.85]
            plot[smooth, variable=\x, domain=-1.5:1.5, samples=300]
                (\x, {0.25*exp(-36*(\x-1)*(\x-1))   + 0.25*exp(-36*(\x+1)*(\x+1))
                    + 0.25*exp(-18*(\x-0.3)*(\x-0.3)) + 0.25*exp(-18*(\x+0.3)*(\x+0.3))})
            -- (1.5, 0) -- (-1.5, 0) -- cycle;
        \draw[baselavline, line width=0.5pt]
            plot[smooth, variable=\x, domain=-1.5:1.5, samples=300]
                (\x, {0.25*exp(-36*(\x-1)*(\x-1))   + 0.25*exp(-36*(\x+1)*(\x+1))
                    + 0.25*exp(-18*(\x-0.3)*(\x-0.3)) + 0.25*exp(-18*(\x+0.3)*(\x+0.3))});

        \fill[ftpurplefill, opacity=0.55]
            plot[smooth, variable=\x, domain=-1.5:1.5, samples=320]
                (\x, {0.3222
                    * (0.25*exp(-36*(\x-1)*(\x-1))   + 0.25*exp(-36*(\x+1)*(\x+1))
                     + 0.25*exp(-18*(\x-0.3)*(\x-0.3)) + 0.25*exp(-18*(\x+0.3)*(\x+0.3)))
                    * exp(3 * (0.2 + 0.3*exp(-8*(\x-0.7)*(\x-0.7))
                                  + 0.4*exp(-6*(\x+0.75)*(\x+0.75))))})
            -- (1.5, 0) -- (-1.5, 0) -- cycle;
        \draw[ftpurpleline, line width=0.7pt]
            plot[smooth, variable=\x, domain=-1.5:1.5, samples=320]
                (\x, {0.3222
                    * (0.25*exp(-36*(\x-1)*(\x-1))   + 0.25*exp(-36*(\x+1)*(\x+1))
                     + 0.25*exp(-18*(\x-0.3)*(\x-0.3)) + 0.25*exp(-18*(\x+0.3)*(\x+0.3)))
                    * exp(3 * (0.2 + 0.3*exp(-8*(\x-0.7)*(\x-0.7))
                                  + 0.4*exp(-6*(\x+0.75)*(\x+0.75))))});

        \draw[ftpurpleline, line width=0.5pt, -{Stealth[length=1.0pt, width=0.8pt]}]
            (-1.05, 0.237) -- (-1.05, 0.269);
        \draw[ftpurpleline, line width=0.85pt, -{Stealth[length=2.6pt, width=2.1pt]}]
            (-0.95, 0.237) -- (-0.95, 0.346);
        \draw[ftpurpleline, line width=0.5pt, -{Stealth[length=1.0pt, width=0.8pt]}]
            (-0.30, 0.246) -- (-0.30, 0.216);
        \draw[ftpurpleline, line width=0.65pt, -{Stealth[length=1.6pt, width=1.3pt]}]
            (-0.22, 0.226) -- (-0.22, 0.175);
        \draw[ftpurpleline, line width=0.7pt, -{Stealth[length=1.9pt, width=1.5pt]}]
            ( 0.22, 0.226) -- ( 0.22, 0.162);
        \draw[ftpurpleline, line width=0.65pt, -{Stealth[length=1.6pt, width=1.3pt]}]
            ( 0.30, 0.246) -- ( 0.30, 0.195);
        \draw[ftpurpleline, line width=0.5pt, -{Stealth[length=1.0pt, width=0.8pt]}]
            ( 1.05, 0.227) -- ( 1.05, 0.196);
        \draw[ftpurpleline, line width=0.55pt, -{Stealth[length=1.1pt, width=0.9pt]}]
            ( 1.10, 0.180) -- ( 1.10, 0.145);
    \end{scope}

    \node[anchor=north, font=\scriptsize, ftpurpleline] at (0, -0.05)
        {reweighting within existing support};
\end{scope}

\end{tikzpicture}
\caption{\textbf{Overview.}
	\textit{(Left)} Our Wasserstein-tilted approach transports each base sample locally to improve its reward, with the displacement magnitude regularized by the transport cost $\Tprior$.
	\textit{(Right)} Standard reward tilting targets $\tilde\rho_1 \propto e^{\lambda r}\rho_1$, which keeps mode positions fixed and only reweights mass within them.}
\label{fig:overview}
\end{figure}
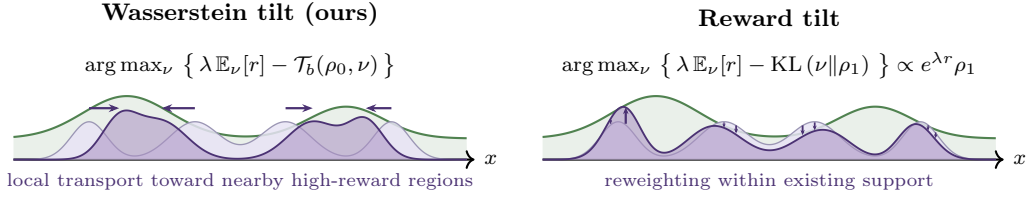

To construct such a framework, we propose to measure deviation from the base distribution via an alternative optimal transport cost built directly from the pre-trained drift.
Unlike KL reward tilting, the resulting objective transports individual samples toward higher reward rather than reweighting the base distribution.
This choice is natural for a deterministic flow because the cost measures how much the dynamics must be steered away from the pre-trained drift.
We show that the resulting fine-tuning problem admits an equivalent deterministic optimal control formulation.
Leveraging this perspective, we arrive at a new class of efficient algorithms in which the flow map emerges as a core component, enabling \textit{simulation-free} backpropagation through model rollouts by amortizing model inference.
We call the resulting framework Wasserstein-tilted flow maps (WTF).

Our \textbf{main contributions} are:
\begin{enumerate}[leftmargin=*, itemsep=2pt, topsep=2pt]
	\item
	      We introduce an optimal transport regularizer between the prior and the candidate terminal distribution that we build directly from the pre-trained drift.
	      We show that its induced fine-tuning objective decouples per particle, moving each sample toward higher reward in contrast to the \emph{reweighting} performed by the reward tilt.

	\item
	      We prove that fine-tuning under our regularizer is equivalent to a deterministic optimal control problem with terminal reward, and identify it as the rescaled small-noise limit of stochastic optimal control-based algorithms for reward tilting.
	\item We devise the first end-to-end fine-tuning algorithm for flow maps.
	      Our method exploits the flow map's ability to advance the dynamics in a single network evaluation.
          In particular, we introduce a Monte Carlo estimate of the deterministic value function, so that terminal samples and value estimates require $O(1)$ flow map evaluations instead of $O(T)$ network evaluations for a $T$-step rollout.  
	      The Monte Carlo value estimator is unbiased and critic-free, and the fine-tuned model retains the flow map's few-step inference at deployment.

	\item %
	      We evaluate WTF on high-resolution class-conditional image synthesis on ImageNet-256 and text-to-image modeling with the TiM-T2I checkpoint~\citep{wang2025transition}. Across both benchmarks, WTF matches or surpasses baselines on reward. At text-to-image scale, WTF reaches matched reward with up to $280\times$ and $47\times$ less compute than zeroth-order and first-order baselines, respectively.
\end{enumerate}

\section{Background}
\label{sec:background}

\subsection{Flow and flow map-based generative models}
\label{ssec:bg_flow_matching}
In this work, we assume access to a pre-trained flow-based generative model specified by a velocity field $b : [0, 1] \times \R^d \to \R^d$.
The resulting probability flow is given by
\begin{equation}
	\label{eqn:pflow}
	\dot{x}_t = b_t(x_t), \qquad x_0 \sim \rho_0,
\end{equation}
which transports samples from a tractable prior distribution $\rho_0 = \Normal(0, \Id)$ to the pre-trained distribution $\rho_1$.
The velocity $b_t$ is typically learned over a stochastic interpolant such as $I_t = (1-t)x_0 + t x_1$ with $x_0 \sim \rho_0$ and $x_1 \sim \rho_1$ by minimizing the flow matching objective~\citep{lipman2022flow, liu2022flow, albergo2023stochastic, albergo2022building}, but can also be obtained from the probability flow for a diffusion model~\citep{song2020score,boffi_probability_2023,karras_elucidating_2022}.
Samples from the pre-trained $\rho_1$ can be obtained by solving~\cref{eqn:pflow} numerically, which typically requires tens to hundreds of network evaluations and is computationally demanding~\citep{karras_elucidating_2022}.

To avoid the cost of this numerical integration, recent work has centered on learning the \textit{flow map} $X : [0, 1]^2 \times \R^d \to \R^d$~\citep{boffi_flow_2024, boffi2025self, song2023consistencymodels, kim_consistency_2024, geng2025meanflowsonestepgenerative, frans2025stepdiffusionshortcutmodels}, which is the solution operator of~\cref{eqn:pflow}.
By definition, the flow map satisfies the \textit{jump condition} along trajectories for any $(s, t) \in [0, 1]^2$,
\begin{equation}
	\label{eqn:jump_cond}
	X_{s, t}(x_s) = x_t.
\end{equation}
Given $X_{s, t}$, we can generate a sample from $\rho_1$ in a single evaluation $x_1 = X_{0, 1}(x_0)$.
More generally, we can use the semigroup property $X_{s, t} = X_{u, t} \circ X_{s, u}$ to implement higher-accuracy multi-step sampling on an arbitrary grid~\citep{boffi_flow_2024,boffi2025self}.
Recent applications of this approach obtain samples in $\leq 4$ steps that match the performance of many-step flows~\citep{geng2025meanflowsonestepgenerative,sabour2025align,zhou2025terminal}.

A complete self-contained review on both flows and flow maps is provided in~\cref{app:flow_maps}.

\subsection{Reward fine-tuning}
\label{ssec:bg_reward_tilt}
Existing approaches for reward fine-tuning and alignment of flow- and diffusion-based generative models take the reward-tilted distribution with an inverse temperature $\lambda > 0$
\begin{equation}
	\label{eqn:reward_tilt}
	\tilde{\rho}_1(x) = \frac{1}{Z}e^{\lambda\, r(x)}\, \rho_1(x), \quad Z = \int e^{\lambda r(x)}\rho_1(x)dx,
\end{equation}
as the fundamental object of interest, and design algorithms to approximately sample it~\citep{clark2024directlyfinetuningdiffusionmodels, domingoenrich2025adjointmatchingfinetuningflow, uehara2025inferencetimealignmentdiffusionmodels,potaptchik2026metaflowmapsenable,holderrieth2025diamond,holderrieth2025glassflowstransitionsampling}.

A classical result based on the Doob $h$-transform~\citep{Doob,karatzas2014brownian} establishes that one can sample from $\tilde{\rho}_1$ by estimating the gradient of a stochastic value function $U:[0, 1]\times \R^d\to\R$,
\begin{equation}
	\label{eqn:bg_value_fn}
	U_t(x) = \log \E\!\left[\, e^{\lambda\, r(X_1)} \,\big|\, X_t = x\, \right],
\end{equation}
where the conditional expectation is taken over a \emph{memoryless} stochastic process
\begin{equation}
	\label{eqn:bg_memoryless_sde}
	dX_t = b_t(X_t) dt + \tfrac{1}{2}\sigma^2(t)\, \nabla\log\rho_t(X_t) dt + \sigma(t) dW_t,
\end{equation}
with $\sigma(t) = \sqrt{2(1-t)/t}$ constructed so that $X_0$ and $X_1$ are independent~\citep{domingoenrich2025adjointmatchingfinetuningflow}.
In~\cref{eqn:bg_memoryless_sde}, $\rho_t = \Law(x_t)$ is the density of the probability flow~\cref{eqn:pflow}.
Given this value function, the modified stochastic process
\begin{equation}
	\label{eqn:controlled_sde}
	d\tilde{X}_t = b_t(\tilde{X}_t) dt + \tfrac{1}{2}\sigma^2(t) \nabla\log\rho_t(\tilde{X}_t) dt + \sigma^2(t)\nabla U_t(\tilde{X}_t)dt + \sigma(t) dW_t,
\end{equation}
then satisfies that $\tilde{X}_1 \sim \tilde{\rho}_1$ is a sample from the reward-tilted measure~\cref{eqn:reward_tilt}.
The fine-tuned diffusion process~\cref{eqn:controlled_sde} can then be converted back into an equivalent fine-tuned flow,
\begin{equation}
	\label{eqn:controlled_flow}
	\dot{\tilde{x}}_t = b_t(\tilde{x}_t) + \tfrac{1}{2}\sigma^2(t)\, \nabla U_t(\tilde{x}_t), \:\: \tilde{x}_0 \sim \rho_0,
\end{equation}
whose terminal law $\Law(\tilde{x}_1) = \tilde{\rho}_1$ is also the reward-tilted measure~\cref{eqn:reward_tilt}.
We refer the reader to~\cref{app:soc_background} for the full derivation of these processes and why they sample from the reward tilt.

\section{Wasserstein-tilted flow maps}
\label{sec:methods}
While conceptually elegant, the above recipe is clunky in practice, as it revolves around fine-tuning a diffusion as a surrogate for the flow.
Computationally, the value gradient $\nabla_x U_t(x)$ depends on a terminal-time conditional expectation, which requires expensive rollouts of the auxiliary process~\cref{eqn:bg_memoryless_sde} all the way to $t = 1$.
These rollouts are slow because stable integration requires many small steps, and they can be high variance because the exponential $e^{\lambda r(X_1)}$ can be dominated by a small number of samples with high reward.
To address these pathologies, we develop a fine-tuning procedure that operates directly on the flow and uses the flow map to amortize deterministic value gradient rollouts. \Cref{fig:rollout_cost} compares the resulting value estimators.
Throughout this section, all stated results assume the standard regularity conditions on $r$, $b$, and $\rho_0$ collected in~\cref{app:standing_assumptions}.

\subsection{An optimal transport regularizer}
\label{ssec:from-regularizer-to-OT}
At a qualitative level, reward fine-tuning aims to find a new terminal distribution $\tilde{\rho}_1$ that maximizes the expected reward while ``staying close'' to the pre-trained flow.
This can be made precise as the regularized reward-maximization problem
\begin{equation}
	\label{eqn:wtf-abstract}
	\tilde{\rho}_1 = \argmax_{\nu\in\calP(\R^d)}\ \lambda\, \E_{x \sim \nu}[r(x)] - \calR(\nu),
\end{equation}
where $\nu$ ranges over candidate terminal distributions on $\R^d$ and $\calR:\calP(\R^d)\to\R_{\geq 0}$ is a regularizer that penalizes the departure of $\nu$ from the base.
Intuitively, the choice of $\calR$ determines our definition of ``close''.
Setting $\calR(\nu) = \kl{\nu}{\rho_1}$ recovers the reward-tilted distribution~\cref{eqn:reward_tilt} of~\cref{ssec:bg_reward_tilt}. Setting $\calR$ to the entropic Schr\"odinger-bridge cost against the auxiliary process~\cref{eqn:bg_memoryless_sde} recovers the approach of~\citet{domingoenrich2025adjointmatchingfinetuningflow}. \Cref{app:sb_zero_noise} gives this connection in detail.
Existing algorithms for these choices are stochastic and rely on a diffusion process.

We instead seek a regularizer defined directly by the deterministic flow.

As the flow's primitive operation is transport, we argue that it is natural to search for $\calR$ over optimal-transport problems.
The general Kantorovich functional~\citep{villani2009optimal} with cost $c:\R^d\times\R^d\to\R_{\geq 0}$,
\begin{equation}
	\label{eqn:kantorovich-general}
	\mathcal{T}_c(\mu, \nu) = \inf_{\pi \in \Pi(\mu, \nu)} \int c(x, y)\, \pi(dx\, dy),
\end{equation}
defines a family of candidate regularizers $\calR(\nu) = \mathcal{T}_c(\mu, \nu)$ parameterized by the cost $c$ and the source measure $\mu$, where $\Pi(\mu, \nu)$ denotes the set of couplings with marginals $\mu$ and $\nu$.

\paragraph{Defining the regularizer.}
The choice $c(x, y) = \tfrac{1}{2}\norm{x - y}^2$ defines the standard $W_2$ distance~\citep{villani2009optimal}, but neither $\calR(\nu) = \tfrac{1}{2} W_2^2(\rho_1, \nu)$ nor $\calR(\nu) = \tfrac{1}{2} W_2^2(\rho_0, \nu)$ is suitable for our goals.
In the former case, the optimizer is a transport map between $\rho_1$ and $\nu$ that must be composed with the base model at inference, leaving the fine-tuned model a two-stage object.
In the latter, the cost knows nothing about $b$ or $\rho_1$.
We therefore propose a regularizer $\calR(\nu) = \mathcal{T}_{\cprior}(\rho_0, \nu)$, where $\cprior$ measures the \textit{residual control} required to \textit{steer} the base flow from an initial point $x\in\R^d$ to a target point $y\in\R^d$.
The source $\rho_0$ ensures that the optimizer is a flow from the same prior as the base model.

For a candidate velocity $v$, we define a Lagrangian $\Lprior: [0, 1]\times \R^d\times\R^d\to\R_{\geq 0}$
to measure the instantaneous deviation of $v$ from the base drift.
We then define our cost function $c_b:\R^d\times\R^d\to\R_{\geq 0}$ as the integrated Lagrangian over paths $\omega : [0, 1] \to \R^d$ joining $x$ to $y$,
\begin{equation}
	\label{eqn:cb-cost}
	\Lprior(t, x, v) = \tfrac{1}{2}\norm{v - b_t(x)}^2, \quad
	\cprior(x, y) = \inf_{\omega_0 = x,\, \omega_1 = y}\ \int_0^1 \Lprior(t, \omega_t, \dot{\omega}_t)\, dt,
\end{equation}
which measures the minimum residual energy required to steer a particle from $x$ to $y$ while following $b_t$ as closely as possible.
We then take our regularizer to be the corresponding Kantorovich problem,
\begin{lavenderbox}
	\begin{equation}
		\label{eqn:Tb-static}
		R(\nu) = \Tprior(\rho_0, \nu) = \inf_{\pi \in \Pi(\rho_0,\, \nu)} \int \cprior(x, y)\, \pi(dx\, dy),
	\end{equation}
\end{lavenderbox}
which defines the minimum total residual energy needed to steer the base law $\rho_0$ to a candidate terminal law $\nu$ along the dynamics $b_t$.
When $b \equiv 0$, $\cprior(x, y) = \tfrac{1}{2}\norm{x - y}^2$ and $\Tprior(\rho_0, \nu) = \tfrac{1}{2} W_2^2(\rho_0, \nu)$, so that $\Tprior$ generalizes the Benamou--Brenier formulation of $W_2^2$~\citep{benamou2000computational} to a non-trivial reference.
This construction has been studied under the name \emph{optimal transport with prior} by~\citet{chen2017OTlineardynamics,chen2016SBcontrol,chen2021stochastic}; we adapt their framework here to fine-tuning of deterministic generative flows.

\begin{figure}[tb]
\centering
\hspace*{-1.2cm}
\resizebox{\linewidth}{!}{\usetikzlibrary{arrows.meta, decorations.pathmorphing, positioning, calc, shapes.geometric}

\definecolor{rolloutgray}{RGB}{135, 135, 145}
\definecolor{rolloutdark}{RGB}{75, 75, 85}
\definecolor{controlblue}{RGB}{72, 110, 180}
\definecolor{controlbluedark}{RGB}{40, 75, 145}
\definecolor{ctrlamber}{RGB}{198, 125, 28}
\definecolor{rewardgreen}{RGB}{60, 120, 80}
\definecolor{ftpurple}{RGB}{75, 50, 115}
\definecolor{criticpurple}{RGB}{135, 100, 165}
\definecolor{criticpurpledark}{RGB}{90, 65, 125}
\definecolor{biaspink}{RGB}{195, 110, 110}
\definecolor{goodgreen}{RGB}{75, 155, 95}
\definecolor{dividergray}{RGB}{200, 200, 210}

\begin{tikzpicture}[
    x=0.85cm, y=0.80cm,
    every node/.style={font=\normalsize},
    dot/.style={circle, fill=ftpurple, draw=white, line width=0.5pt,
                inner sep=0pt, minimum size=4.5pt},
    enddot/.style={circle, fill=black, draw=white, line width=0.5pt,
                   inner sep=0pt, minimum size=5pt},
    fmlabel/.style={fill=white, draw=controlbluedark, line width=0.7pt,
                    inner sep=2.5pt, font=\normalsize\bfseries, text=controlbluedark,
                    rounded corners=2pt},
    netbox/.style={rectangle, draw=criticpurpledark, fill=criticpurple!18,
                   line width=0.8pt, rounded corners=3pt, inner sep=7pt,
                   font=\normalsize\bfseries, text=criticpurpledark,
                   minimum width=24pt, minimum height=24pt, align=center},
    pillgood/.style={rectangle, fill=goodgreen!18, draw=goodgreen!75, line width=0.5pt,
                     rounded corners=4pt, inner xsep=3pt, inner ysep=2pt,
                     text height=1.6ex, text depth=0.5ex,
                     minimum height=2.6ex, minimum width=1.5cm,
                     font=\footnotesize\bfseries, text=goodgreen!50!black},
    pillbad/.style={rectangle, fill=biaspink!18, draw=biaspink, line width=0.5pt,
                    rounded corners=4pt, inner xsep=3pt, inner ysep=2pt,
                    text height=1.6ex, text depth=0.5ex,
                    minimum height=2.6ex, minimum width=1.5cm,
                    font=\footnotesize\bfseries, text=biaspink!60!black},
]

\draw[dividergray, dashed, line width=0.5pt] (7.5, 1.05) -- (7.5, 4.6);
\draw[dividergray, dashed, line width=0.5pt] (15.5, 1.05) -- (15.5, 4.6);

\begin{scope}[shift={(0, 0)}]

\node[anchor=south, font=\large\bfseries] at (3.15, 4.35) {(a) Actor--Critic};

\node[dot] (xt_a) at (0.7, 1.85) {};
\node[anchor=east, xshift=-4pt, font=\normalsize, ftpurple] at (xt_a) {$x_t$};

\draw[criticpurpledark, line width=1.3pt, -{Stealth[length=5pt]}]
    (xt_a) -- (2.3, 1.85);

\node[netbox, anchor=center] (critic) at (3.15, 1.85)
    {$V_\theta(t,x)$};

\draw[criticpurpledark, line width=1.3pt, -{Stealth[length=5pt]}]
    (critic.east) -- (5.4, 1.85);
\node[enddot] (vout) at (5.6, 1.85) {};
\node[anchor=west, xshift=4pt, font=\normalsize, criticpurpledark]
    at (vout) {$\widehat V_t$};

\draw[ctrlamber, line width=1pt, dashed, -{Stealth[length=4.5pt]}]
    (5.6, 2.55) .. controls (4.0, 3.30) and (2.3, 3.30) .. (0.7, 2.55);
\node[font=\small, ctrlamber, anchor=south] at (3.15, 3.40)
    {backprop $\nabla_{x_t} V_\theta$};

\node[font=\small, criticpurpledark, anchor=north] at (3.15, 1.30)
    {learned critic network};

\node[pillgood, anchor=north] at (0.7, 0.62) {$O(1)$};
\node[pillbad, anchor=north]  at (3.15, 0.62) {biased};
\node[pillbad, anchor=north]  at (5.6, 0.62) {needs critic};

\end{scope}

\begin{scope}[shift={(8, 0)}]

\node[anchor=south, font=\large\bfseries] at (3.15, 4.35) {(b) Multi-step rollout};

\def\xtPos{0.7}
\def\xRPos{5.6}
\draw[rolloutgray, line width=1.5pt, smooth, samples=80, domain=0:1, variable=\u,
      -{Stealth[length=4pt, width=3.5pt]}]
    plot ({0.7 + 4.9*\u}, {1.85 - 0.32*sin(540*\u)});
\foreach \i in {1,...,11} {
    \pgfmathsetmacro{\u}{\i/12}
    \node[circle, fill=rolloutgray, draw=white, line width=0.4pt, inner sep=0pt, minimum size=3.6pt]
        at ({0.7 + 4.9*\u}, {1.85 - 0.32*sin(540*\u)}) {};
}

\node[dot] (xt_b) at (\xtPos, 1.85) {};
\node[anchor=east, xshift=-4pt, font=\normalsize, ftpurple] at (xt_b) {$x_t$};
\node[enddot] (x1_b) at (\xRPos, 1.85) {};
\node[anchor=west, xshift=4pt, font=\normalsize, rewardgreen] at (x1_b) {$r(x_1)$};

\draw[ctrlamber, line width=1pt, dashed, -{Stealth[length=4.5pt]}]
    (5.6, 2.55) .. controls (4.0, 3.30) and (2.3, 3.30) .. (0.7, 2.55);
\node[font=\small, ctrlamber, anchor=south] at (3.15, 3.40)
    {backprop $\times T$};

\node[font=\small, rolloutdark, anchor=north] at (3.15, 1.30)
    {$T$ network calls};

\node[pillgood, anchor=north] at (0.55, 0.62) {unbiased};
\node[pillbad, anchor=north]  at (3.0, 0.62) {$O(T)$ rollout};
\node[pillbad, anchor=north]  at (5.85, 0.62) {$O(T)$ memory};

\end{scope}

\begin{scope}[shift={(16, 0)}]

\node[anchor=south, font=\large\bfseries, controlbluedark] at (3.15, 4.35)
    {(c) WTF (ours)};

\draw[rolloutgray!75, line width=1.1pt, dash pattern=on 3pt off 2.5pt, smooth, samples=80,
      domain=0:1, variable=\u]
    plot ({0.7 + 4.9*\u}, {1.85 - 0.32*sin(540*\u)});

\draw[controlblue, line width=2.1pt, line cap=round,
      -{Stealth[length=8pt, width=6.5pt, inset=2.5pt]}]
    (0.7, 1.85) .. controls (1.45, 2.42) and (2.45, 2.55) .. (3.12, 2.20);
\node[fmlabel] at (1.80, 2.56) {$X^{\bar u}_{t,\tau}$};

\draw[controlblue, line width=2.1pt, line cap=round,
      -{Stealth[length=8pt, width=6.5pt, inset=2.5pt]}]
    (3.18, 2.20) .. controls (3.85, 2.55) and (4.85, 2.42) .. (5.6, 1.85);
\node[fmlabel] at (4.50, 2.56) {$X^{\bar u}_{\tau,1}$};

\node[dot] (xt_c) at (0.7, 1.85) {};
\node[anchor=east, xshift=-4pt, font=\normalsize, ftpurple] at (xt_c) {$x_t$};

\node[dot] (xtau_c) at (3.15, 2.17) {};

\draw[ctrlamber, line width=1.3pt, -{Stealth[length=3.5pt, width=2.8pt]}]
    (xtau_c) -- ++(0.0, 0.34);
\node[font=\small, ctrlamber, anchor=south, yshift=-1pt]
    at ($(xtau_c) + (0.0, 0.34)$) {$\bar u_\tau$};

\node[anchor=north, yshift=-4pt, font=\small, ftpurple] at (xtau_c) {$x_\tau$};

\node[enddot] (x1_c) at (5.6, 1.85) {};
\node[anchor=west, xshift=4pt, font=\normalsize, rewardgreen] at (x1_c) {$r(x_1)$};

\draw[ctrlamber, line width=1pt, dashed, -{Stealth[length=4.5pt]}]
    (5.6, 2.95) .. controls (4.0, 3.75) and (2.3, 3.75) .. (0.7, 2.95);
\node[font=\small, ctrlamber, anchor=south] at (3.15, 3.82)
    {single Jacobian backprop};

\node[font=\small, controlbluedark, anchor=north] at (3.15, 1.30)
    {$2$ flow map calls};

\node[pillgood, anchor=north] at (0.7, 0.62) {unbiased};
\node[pillgood, anchor=north] at (3.15, 0.62) {$O(1)$};
\node[pillgood, anchor=north] at (5.6, 0.62) {no critic};

\end{scope}

\end{tikzpicture}}
\caption{\textbf{Flow map value estimation.}
  Common alternatives use either a learned critic, as in actor--critic methods, or a multi-step rollout that backpropagates through $T$ steps.
  WTF instead reaches $x_1$ from $x_t$ in two flow map evaluations and a single Jacobian backprop, which is $O(1)$, unbiased, and does not require a critic.
  }
\label{fig:rollout_cost}
\vspace{-0.5\baselineskip}
\end{figure}
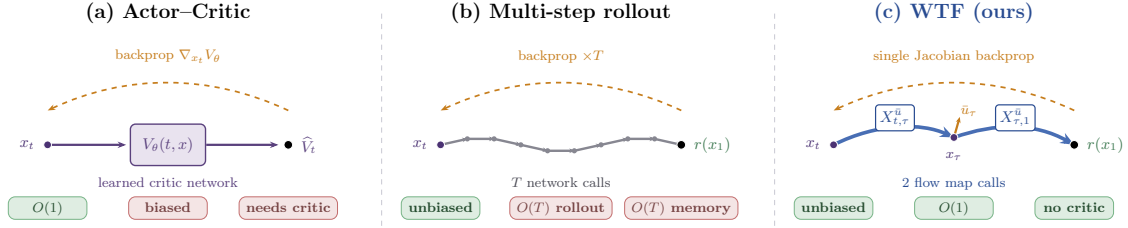

\subsection{Understanding the regularizer}
\label{ssec:understanding-regularizer}
Combining $\Tprior$ with the reward function $r$ gives the regularized reward maximization problem
\begin{equation}
	\label{eqn:wtf-reward-objective}
	\tilde{\rho}_1 = \argmax_{\nu \in \calP(\R^d)}\ \left\{\ \lambda\, \E_{x \sim \nu}[r(x)] - \Tprior(\rho_0, \nu)\ \right\},
\end{equation}
whose optimizer $\tilde{\rho}_1$ defines a \emph{Wasserstein-tilted flow} (WTF).
To understand the nature of this problem, we examine how the distribution-level optimization acts on individual source samples.
The following result shows that the optimum can be understood as transporting each source sample toward a high-reward endpoint, subject to the cost of deviating from the pre-trained flow.

\begin{restatable}{proposition}{ppdecomp}
	\label{prop:pointwise-decomp}
	Under the regularity conditions of~\cref{app:standing_assumptions},
	
	\begin{equation}
		\label{eqn:pointwise-decomp}
		\sup_{\nu}\ \left\{\ \lambda\, \E_{x \sim \nu}[r(x)] - \Tprior(\rho_0, \nu)\ \right\}
		\;=\;
		\E_{x \sim \rho_0}\!\left[\, \sup_{y \in \R^d}\ \bigl\{\, \lambda\, r(y) - \cprior(x, y)\, \bigr\}\, \right],
	\end{equation}
	and the optimum is attained by choosing $y^*(x) = \argmax_y \{\, \lambda\, r(y) - \cprior(x, y)\, \}$ for each $x \sim \rho_0$.
\end{restatable}

The proof, given in~\cref{app:pointwise_decomp}, follows from Kantorovich duality~\citep{villani2009optimal,santambrogio2015optimal}.
The result gives a concrete particle-level interpretation of the objective.
The optimal solution transports each source sample to the destination that gives its best reward-cost tradeoff, rather than selecting the terminal law only through global reweighting.
\citet{moitra2026tractabilitylandscapediffusionalignment} give the analogous per-sample reduction for $W_2^2$ regularization toward the base law, where the pointwise problem is a proximal transport oracle; Proposition~\ref{prop:pointwise-decomp} extends it to the drift-induced cost $c_b$ with source $\rho_0$.

\paragraph{Comparison with KL reward tilting.}
The sample-level distinction above reflects a broader structural difference between the two objectives.
KL reward tilting reweights the base distribution according to reward, whereas WTF transports samples under a cost induced by the pre-trained drift.
As a consequence, reward reweighting preserves the base conditional distribution within each reward level set, while WTF can redistribute mass within a level set.
\Cref{app:synthetic} gives exact examples of this distinction.

\enlargethispage{\baselineskip}
While our proposed regularizer~\cref{eqn:Tb-static} is natural for the problem under study, and while~\cref{prop:pointwise-decomp} demonstrates a clean interpretation of what it does at the sample level, it is still unclear how to implement it algorithmically.
The following result shows that~\cref{eqn:wtf-reward-objective} can be reformulated as a deterministic optimal control problem with terminal reward, which will enable us to develop scalable algorithms.
\begin{restatable}{proposition}{wassersteintheorem}
	\label{thm:wasserstein_tilt}
	Under standard regularity conditions on $r$, $b$, and $\rho_0$, the optimal value of~\cref{eqn:wtf-reward-objective} equals the optimal value of the deterministic optimal control problem
	\begin{lavenderbox}
		\begin{equation}
			\label{eqn:oc_problem}
			\begin{aligned}
				 & \sup_{u}\ \E_{x_0 \sim \rho_0}\!\left[\, \lambda\, r(x_1^u) - \frac{1}{2} \int_0^1 \norm{u_t(x_t^u)}^2\, dt\, \right], \\
				 & \text{subject to} \quad \dot{x}_t^u = b_t(x_t^u) + u_t(x_t^u),\quad x_0^u = x_0.
			\end{aligned}
		\end{equation}
	\end{lavenderbox}
	Whenever an optimizer $u^*:[0, 1]\times\R^d\to\R^d$ exists, the terminal law $\rho^{u^*}_1 = \Law(x_1^{u^*})$ solves~\cref{eqn:wtf-reward-objective}.
\end{restatable}
The proof is given in~\cref{app:wasserstein_tilt}, which proceeds via a Benamou-Brenier-style dynamic reformulation of the static regularizer~\cref{eqn:Tb-static} established by~\citet{chen2017OTlineardynamics,chen2016SBcontrol,chen2021stochastic}.
A direct consequence is that the fine-tuned model is again a single drift $\tilde{v}_t = b_t + u^*_t$ that replaces $b_t$ at inference.
We further show in~\cref{app:sb_zero_noise} that the problem~\cref{eqn:oc_problem} is the zero-noise limit of a Schr\"odinger-bridge problem~\citep{leonard2014survey,chen2016SBcontrol,chen2017OTlineardynamics,chen2021stochastic}, connecting WTF to the stochastic optimal control framework used by Adjoint Matching and related methods~\citep{domingoenrich2025adjointmatchingfinetuningflow,potaptchik2026metaflowmapsenable,holderrieth2025diamond,holderrieth2025glassflowstransitionsampling}. \Cref{tab:regularizer-comparison} situates the corresponding regularizers and algorithms.

\section{Solving the optimal control problem}
\label{sec:training}
We now develop a practical algorithm to solve the optimal control problem~\cref{eqn:oc_problem} using a pre-trained flow map $X_{s, t}(x) = x + (t - s)\, v_{s, t}(x)$ for the uncontrolled dynamics $\dot{x}_t = b_t(x_t)$ with $b_t = v_{t, t}$.
In practice, we find that direct backpropagation through a sampled version of~\cref{eqn:oc_problem} exhibits a reward-hacking failure mode, where the terminal map can increase reward while the diagonal velocity fails to learn the corresponding controlled dynamics; see~\cref{app:naive_fail} for details.
To avoid this pathology, we instead regress the diagonal control onto the value gradient of a frozen reference control, giving a policy-improvement update.
The flow map's long-range transitions let us construct the on-policy states and value estimates required by this update without numerically simulating every intermediate state, yielding simulation-free rollouts and an end-to-end fine-tuning algorithm for flow maps.
For a control $u$, define the value function as the reward-to-go from state $x$ at time $t$ minus the remaining control cost,
\begin{equation}
	\label{eqn:value_fn}
	V^u_t(x) = \lambda\, r(X^u_{t, 1}(x)) - \frac{1}{2}\int_t^1 \norm{u_\tau(X^u_{t, \tau}(x))}^2\, d\tau,
\end{equation}
so that maximizing $\E_{x_0}[V^u(0, x_0)]$ recovers~\cref{eqn:oc_problem}.
Thus $V^u_t(x)$ is the future objective for a trajectory starting from $x$ at time $t$; \cref{app:doc_background} provides a self-contained introduction to deterministic optimal control and the role of the value function.
A standard result in optimal control theory~\citep{liberzon2011calculus} states that the optimal control for~\cref{eqn:oc_problem} is the gradient of the optimal value function,
\begin{equation}
	\label{eqn:optimal_control}
	u^*_t(x) = \nabla_x V^*_t(x), \quad V_t^*(x) = \sup_u V_t^u(x).
\end{equation}
The following result shows that fitting our trainable control to the current estimate of the value gradient in an iterative fashion is guaranteed to improve the objective at the population level.

\begin{restatable}[Performance difference and policy improvement]{proposition}{performancediff}
	\label{prop:performance_diff}
	Let $\bar{u}$ be a reference control with value function $V^{\bar{u}}_t$ and value gradient $\bar{g}_t(x) = \nabla_x V^{\bar{u}}_t(x)$.
	For any candidate control $w$ and any initial pair $(s, x)$ with controlled trajectory $x^w_t = X^w_{s, t}(x)$ for $t \in [s, 1]$,
	\begin{equation}
		\label{eqn:performance_diff}
		V^w_s(x) - V^{\bar{u}}_s(x) = \frac{1}{2} \int_s^1 \!\left[\, \norm{\bar{u}_t(x^w_t) - \bar{g}_t(x^w_t)}^2 - \norm{w_t(x^w_t) - \bar{g}_t(x^w_t)}^2\, \right] dt.
	\end{equation}
	In particular, setting $w = \bar g$ makes the second norm vanish and gives $V^{\bar g}_s(x) \geq V^{\bar u}_s(x)$ for all $(s, x)$, with equality if and only if $\bar u_t(x^{\bar g}_t) = \bar g_t(x^{\bar g}_t)$; in that case $\bar u$ is the optimal control of~\cref{eqn:oc_problem}.
\end{restatable}

The proof is given in~\cref{app:performance_diff}.
The identity~\cref{eqn:performance_diff} implies that $w = \bar g$, the gradient of the current value function, maximizes the improvement term.
Starting from a reference control $\bar u$, we perform policy iteration by fitting the trainable control $w$ to the corresponding value gradient $\bar g$ and then using $w$ as the reference control for the next iteration.

By~\cref{prop:performance_diff}, each iterate strictly improves the objective.
The iteration continues to improve until it reaches the fixed point $\bar u = \bar g$, which is the global optimum.

\paragraph{Simulation-free value estimation.}
Computing $\nabla_x V^{\bar u}$ for a frozen reference $\bar u$ requires integrating along and then backpropagating through the entire future trajectory, both of which are computationally expensive.
To avoid this, we observe that the integrated control cost can be written as an expectation,
$\int_t^1 \norm{\bar u_\tau(X^{\bar u}_{t, \tau}(x))}^2\, d\tau = (1 - t)\, \E_{\tau \sim \Unif[t, 1]}\!\left[\norm{\bar u_\tau(X^{\bar u}_{t, \tau}(x))}^2\right]$.
Defining $x_\tau = X^{\bar{u}}_{t, \tau}(x)$ and $x_1 = X^{\bar{u}}_{\tau, 1}(x_\tau)$, this gives an unbiased single-sample Monte Carlo estimate of~\cref{eqn:value_fn},
\begin{lavenderbox}
	\begin{equation}
		\label{eqn:mc_value}
		\widehat{V}(t, x) = \lambda\, r(x_1) - \frac{1 - t}{2}\norm{\bar{u}_\tau(x_\tau)}^2.
	\end{equation}
\end{lavenderbox}
The flow map gives the endpoint $x_1$ without numerically simulating the full future trajectory, so the only remaining Monte Carlo approximation is the scalar time integral.
We emphasize that the only variance in~\cref{eqn:mc_value} comes from our Monte Carlo estimate of a one-dimensional time integral, in contrast to the high-dimensional integrals that need to be computed for stochastic value functions arising from reward tilting.
This formula gives an unbiased regression target $\widehat{g}(t, x) = \nabla_x \widehat{V}(t, x)$ via automatic differentiation.
The WTF training objective combines value gradient regression on the diagonal $s = t$ with an off-diagonal flow map self-distillation loss:
\begin{lavenderbox}
	\begin{equation}
		\label{eqn:val_objective}
		\begin{aligned}
			\calL_{\WTF}(\hat{w}) & = \E_{x_0, t, \tau}\!\left[\, \norm{\hat{w}_{t, t}(\bar{x}_t) - b_t(\bar{x}_t) - \widehat{g}(t, \bar{x}_t)}^2\, \right] + \beta\, \calL_{\dist}(\hat{w}), \\
			\bar{w}              & = \sg{\hat{w}}, \quad \bar{x}_t = X^{\bar{w}}_{0, t}(x_0),
		\end{aligned}
	\end{equation}
\end{lavenderbox}
where $\sg{\cdot}$ denotes the stop-gradient operator, $\widehat{g}(t, x) = \nabla_x \widehat{V}(t, x)$ is built from the Monte Carlo estimator~\cref{eqn:mc_value}, and $\calL_{\dist}$ is any of the self-distillation losses~\cref{eqn:app_lsd_esd_psd} reviewed in~\cref{app:flow_maps}.
Here $\bar w = \sg{\hat w}$ is a frozen reference used to form the on-policy states and value-gradient targets.
In addition to ensuring that the output of our method is a fine-tuned flow map, continual self-distillation ensures that we always have access to \textit{amortized on-policy rollouts} for value-function estimation.
At the population level,  we call $\hat w$ a fixed point when~\cref{eqn:val_objective} is stationary in $\hat w$ with $\bar w$ and the stopped targets held fixed, and $\bar w = \hat w$.

\begin{restatable}[Fixed-point optimality of WTF]{proposition}{wtfcriticalpoint}
	\label{prop:wtf_critical_point}
	Under the regularity conditions of~\cref{app:standing_assumptions} and assuming $\rho_0$ has full support on $\R^d$, any population fixed point of~\cref{eqn:val_objective} satisfies $\hat w_{t, t} = b_t + u^*_t$.
	Moreover, $X^{\hat w}_{s, t}$ is the flow map of $b_t + u^*_t$.
\end{restatable}

The proof, given in~\cref{app:wtf_critical_point}, factorizes the fixed-point condition.
The diagonal regression gives $\hat w_{t, t} = b_t + \nabla_x V^{\bar w}_t$ on the on-policy support, while the off-diagonal distillation makes $X^{\hat w}$ the flow map of $\hat w_{t, t}$.
At $\bar w = \hat w$, \cref{prop:performance_diff} then gives the optimal control.

\Cref{alg:wtf} summarizes the training procedure and shows the two choices of reward gradient discussed in~\cref{sec:experiments}.

The synthetic experiments use~\cref{alg:wtf} as written, while the large-scale image experiments use the Euclidean reward-gradient variant described in~\cref{sec:algorithmic}.
Both loss terms are evaluated on every iteration; \cref{app:hyper_t2i} records the practical choice of frozen reference $\bar w$, which is an exponential moving average of $\hat w$ on ImageNet and the stop-gradient copy of the live weights on text-to-image.

\begin{algorithm}[!t]
	\caption{WTF: Wasserstein-tilted fine-tuning of flow maps}
	\label{alg:wtf}
	\KwIn{Pre-trained flow map $X_{s, t}(x) = x + (t - s)\, v_{s, t}(x)$; reward $r$; scale $\lambda$; distillation weight $\beta$; off-diagonal sampler $p_{s, t}$; learning rate $\eta$; reward-gradient type $\in \{\mathrm{exact}, \mathrm{Euclidean}\}$}
	\KwOut{Fine-tuned flow map $X^{\hat w}_{s, t}(x) = x + (t - s)\, \hat w_{s, t}(x)$}
	Initialize $\hat w$ from $v$; set frozen reference $\bar{w} = \sg{\hat w}$ \tcp*{$\bar w$ may be an EMA of $\hat w$; see~\cref{app:hyper_t2i}}
	\Repeat{converged}{
		Sample $x_0 \sim \rho_0$, $t \sim \Unif[0, 1]$, $\tau \sim \Unif[t, 1]$\;
		Forward (frozen $\bar{w}$): $\bar{x}_t = X^{\bar{w}}_{0, t}(x_0)$, $\bar{x}_\tau = X^{\bar{w}}_{t, \tau}(\bar{x}_t)$, $\bar{x}_1 = X^{\bar{w}}_{\tau, 1}(\bar{x}_\tau)$\;
		Monte Carlo value: $\widehat{V} = \lambda\, r(\bar{x}_1) - \tfrac{1 - t}{2}\, \norm{\bar{w}_{\tau, \tau}(\bar{x}_\tau) - b_\tau(\bar{x}_\tau)}^2$\;
		Value gradient: $\widehat{g} = \nabla_{\bar{x}_t} \widehat{V}$ \tcp*{autograd through frozen $\bar w$}
		\uIf{\textnormal{reward-gradient type} $=$ \textnormal{exact}}{
			$\widehat{g}_r \leftarrow \nabla X^{\bar w}_{t, 1}(\bar{x}_t)^\top \nabla r(\bar{x}_1)$\;
		}
		\Else{
			$\widehat{g}_r \leftarrow \nabla r(\bar{x}_1)$\;
		}
		Diagonal loss: $\calL_{\mathrm{val}} = \norm{\hat w_{t, t}(\bar{x}_t) - \sg{\,b_t(\bar{x}_t) + \widehat{g}\,}}^2$\;
		Sample $(s', t') \sim p_{s, t}$ over the upper triangle; compute $\calL_{\dist}$ via~\cref{eqn:dist_eulerian}\;
		Update: $\theta \leftarrow \theta - \eta\, \nabla_\theta\!\left(\calL_{\mathrm{val}} + \beta\, \calL_{\dist}\right)$\;
	}
	\Return $X^{\hat w}_{s, t}(x) = x + (t - s)\, \hat w_{s, t}(x)$\;
\end{algorithm}

\section{Related work}
\label{sec:related_main}
\paragraph{Generative modeling via dynamical transport.}
Diffusion and flow models generate samples by transporting a simple prior to the data distribution through learned stochastic or deterministic dynamics~\citep{lipman2022flow, liu2022flow, albergo2023stochastic, albergo2022building}.
Flow matching learns the deterministic dynamics directly, but sampling still requires numerical integration and repeated network evaluations.
Flow maps amortize this computation by learning the solution operator between pairs of times~\citep{boffi_flow_2024, boffi2025self, song2023consistencymodels, kim_consistency_2024, frans2025stepdiffusionshortcutmodels, geng2025meanflowsonestepgenerative}, enabling generation in one or a few evaluations.
Although developed primarily for accelerated inference, flow maps have recently been used for reward alignment~\citep{potaptchik2026metaflowmapsenable, holderrieth2025diamond, mammadov2026variationalflowmaps}.
We study direct reward fine-tuning of the flow map, so the fine-tuned model retains few-step generation.

\paragraph{Reward fine-tuning.}
Reward fine-tuning adapts a pre-trained generative model to improve a downstream reward while limiting deviation from the base model.
A common formulation is KL-regularized reward maximization, whose optimizer is the exponential reward tilt.
Adjoint Matching~\citep{domingoenrich2025adjointmatchingfinetuningflow} casts this target as memoryless stochastic optimal control, while RAM~\citep{bergmeister2026reinforce} derives a regression objective for the same KL-regularized target.
Reinforcement-learning methods including DDPO~\citep{black2024ddpo}, DPOK~\citep{fan2023dpok}, and Flow-GRPO~\citep{liu2025improving} instead optimize reward through the generative sampler.
VGG-Flow~\citep{liu2025vggflow} also derives flow fine-tuning from deterministic optimal control, but learns a value-gradient critic and updates the velocity field.
WTF instead regularizes with the transport cost induced by the pre-trained deterministic drift.
This yields a deterministic control problem in which flow map transitions provide simulation-free value estimates without a learned critic.
Optimal transport has also appeared in reinforcement learning for generative policies~\citep{sun2025scorediffusionpolicyot}, where transport enters through a critic-based policy objective. In WTF, the reward-independent transport cost is built from the pre-trained dynamics and serves directly as the regularizer.

\enlargethispage{2\baselineskip}

\paragraph{Flow maps for reward alignment.}
Flow maps have also begun to play a direct role in reward-alignment methods.
MFM~\citep{potaptchik2026metaflowmapsenable} uses stochastic flow maps for efficient conditional endpoint sampling and value estimation, but its fine-tuning procedure returns a velocity field and therefore requires distillation for few-step deployment.
VFM~\citep{mammadov2026variationalflowmaps} learns a noise adapter together with a flow map for one-step conditional generation, while score-distillation approaches~\citep{kumari2025npedit, chen2025flashdmdhighfidelityfewstepimage} regularize one-step generators toward a base distribution.
These methods differ in how the flow map enters the alignment procedure and in the model retained after fine-tuning. WTF directly fine-tunes the deterministic flow map across time pairs, preserving inference across NFE budgets and compatibility with subsequent flow map alignment.

\paragraph{Inference-time alignment.}
\looseness=-1
Complementary to fine-tuning, inference-time methods keep model parameters fixed and modify generation through reward guidance, particle-based sampling, or test-time optimization~\citep{uehara2025inferencetimealignmentdiffusionmodels, dhariwal2021diffusion, ho2022classifier, ye2024tfgunifiedtrainingfreeguidance, yu2023freedomtrainingfreeenergyguidedconditional, chung2024diffusionposteriorsamplinggeneral, kim2025flowdps, skreta2025feynmankaccorrectorsdiffusionannealing, singhal2025generalframeworkinferencetimescaling}.
FMRG~\citep{huang2026guideflowfewstepalignment} similarly formulates few-step reward guidance of flow maps as deterministic optimal control.
Diamond Maps~\citep{holderrieth2025diamond} instead uses stochastic flow maps for value estimation and supports guidance, search, and sequential Monte Carlo.
\citet{moitra2026tractabilitylandscapediffusionalignment} show that the choice of closeness measure affects both the computational primitive and the tractable reward class, with Wasserstein closeness reducing to a proximal transport oracle.
Because WTF retains a flow map after fine-tuning, methods such as FMRG can be applied directly for additional inference-time alignment; velocity-field fine-tuning methods require flow map distillation first.

\section{Experiments}
\label{sec:experiments}

\begin{figure}[tb]
    \centering
    \includegraphics[width=\linewidth]{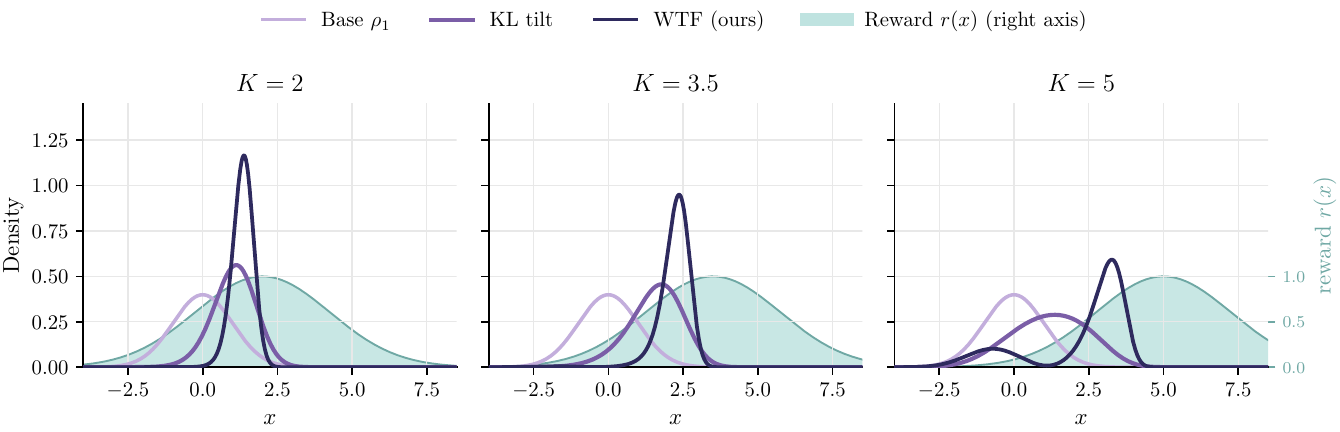}
\caption{\textbf{Comparison of WTF and KL regularization in one dimension.}
Each panel places the reward at a different distance $K$ from the mean of the base law and shows the resulting terminal distributions.
Both terminal laws are computed from their exact population optima.
As the reward moves into lower-density regions, KL tilting reweights the existing mass, whereas WTF transports mass toward the reward.
}
    \label{fig:synth_lowdensity}
\end{figure}

To gain intuition for how the Wasserstein-tilted and KL-tilted objectives differ, we first compare their population optima in one dimension, where both can be evaluated without training.
We then evaluate trained flow maps on ImageNet-256 and text-to-image generation.
The exact reward gradient differentiates the terminal reward through the flow map transition, giving the pullback $\nabla X_{t, 1}(\bar{x}_t)^\top \nabla r(\bar{x}_1)$.
We use this gradient in the synthetic experiments.
For our large-scale image experiments, we instead use $\nabla r(\bar{x}_1)$ directly while leaving the forward endpoint unchanged.
This removes the flow map Jacobian from the reward gradient and computes the update in the direction that increases reward at the generated endpoint.
Following the terminology of FMRG~\citep{huang2026guideflowfewstepalignment}, we refer to the resulting update as the Euclidean reward gradient.
We find that this choice improves both reward and diversity, although the exact gradient also achieves strong results and converges rapidly.
We highlight these choices in~\cref{alg:wtf}.
\Cref{sec:algorithmic} gives implementation details, and \cref{ssec:exp_detach} compares the two empirically.
For ImageNet-256 we fine-tune DMF XL/2~\citep{lee2025decoupled}, and for text-to-image we fine-tune TiM-T2I~\citep{wang2025transition}.
We use HPSv2~\citep{wu2023humanpreferencescorev2}, a learned human-preference score, as the training reward in both settings.
We report PickScore~\citep{kirstain2023pickapicopendatasetuser} and ImageReward~\citep{xu2023imagerewardlearningevaluatinghuman} as additional reward metrics. We measure diversity by the mean pairwise squared distance in DreamSim and CLIP embedding space, which collapses to zero when all samples coincide.
On ImageNet, we compare against Adjoint Matching~\citep{domingoenrich2025adjointmatchingfinetuningflow} and other flow map-based fine-tuning methods, MFM~\citep{potaptchik2026metaflowmapsenable} and VFM~\citep{mammadov2026variationalflowmaps}.
On text-to-image, we compare against Adjoint Matching, a first-order method, and Flow-GRPO~\citep{liu2025improving}, a zeroth-order method.
Because WTF fine-tunes the full flow map across time pairs, the same checkpoint can be evaluated at different NFE budgets without post-hoc distillation.
Full training and evaluation details are in~\cref{sec:algorithmic,app:hyper}.

\subsection{Synthetic experiments}
\label{ssec:exp_synthetic}
We first isolate the effect of the regularizer by comparing the WTF and KL population optima in one dimension.
The KL optimum $\tilde{\rho} \propto \rho_1 e^{\lambda r}$ reweights the base law, while~\cref{prop:pointwise-decomp} gives the WTF optimum by transporting each sample to its reward-cost optimum.
We take $\rho_1 = \Normal(0, 1)$ and use the Gaussian reward bump $r_K(x) = \exp(-(x-K)^2/2w^2)$ centered at $K$, with width $w = 2.25$ and reward scale $\lambda = 7$.
Both terminal laws are evaluated without training.
We compute the KL tilt by numerical quadrature.
For WTF, the prior-action cost is available in closed form, after which we solve the pointwise optimization in~\cref{thm:wasserstein_tilt} and obtain the terminal density by change of variables.
This gives an exact comparison of the two population optima without sampling.
\Cref{fig:synth_lowdensity} shows that WTF achieves higher reward than KL tilting, with the gap increasing as the reward moves into lower-density regions of the base distribution.
This behavior follows from the different regularizers.
KL tilting can only reweight mass already present under the base law, whereas WTF can transport mass toward high-reward regions.
Transport is therefore most advantageous when high-reward regions carry little mass under the base distribution

The two objectives also differ in which terminal laws they can reach. 
A tilt $\nu \propto \rho_1 e^{\lambda r}$ cannot change the relative density of two points with equal reward: if $r(x_a) = r(x_b)$, both are multiplied by the same factor $e^{\lambda r(x_a)}$, so
\begin{equation*}
	\frac{\nu(x_a)}{\nu(x_b)} = \frac{\rho_1(x_a)}{\rho_1(x_b)}
	\qquad \text{for every } \lambda .
\end{equation*}
Transport has no such constraint and can move mass between such points, so WTF can reach terminal laws that no tilt of $\rho_1$ by $r$ produces.

\subsection{ImageNet-256 main results}
\label{ssec:exp_imagenet}
We fine-tune DMF XL/2 with HPSv2 as the reward.
WTF results are means over three matched seeds, and every method was given the same fine-tuning budget of roughly five hours on one $8 \times \mathrm{H100}$ node.

\begin{table}[tb]
    \centering
    \resizebox{\textwidth}{!}{%
        \begin{tabular}{lcccccc}
            \toprule
            & & \multicolumn{3}{c}{Reward $\uparrow$} & \multicolumn{2}{c}{Diversity $\uparrow$} \\
            \cmidrule(lr){3-5} \cmidrule(lr){6-7}
            Method & NFE
                & HPSv2 (fine-tuned)
                & PickScore
                & ImageReward
                & DreamSim
                & CLIP diversity \\
            \midrule
            Base Model         & 1   &  0.214 & 19.21 & $-0.271$ & 0.790 & 0.350 \\
            Base Model         & 250 &  0.218 & 19.45 & $-0.252$ & 0.841 & 0.340 \\
            \midrule
            Adjoint Matching   & 250 &  0.243 & 20.26 &  0.005   & 0.453 & 0.237 \\
            MFM                & 250 &  0.291 & 19.83 &  0.544   & 0.379 & 0.189 \\
            VFM                & 1   &  0.307 & 20.91 &  0.860   & 0.482 & 0.194 \\
            \midrule
            \ourrow WTF (Ours) & 1   &  0.319 & 20.24 &  0.612  & 0.482 & 0.201 \\
            \ourrow WTF (Ours) & 250 &  0.330 & 20.45 &  0.718  & 0.460 & 0.177 \\
            \bottomrule
        \end{tabular}%
    }
    \vspace{0.5em}
    \caption{\label{tab:reward-comparison-imagenet}\textbf{ImageNet results.}
    WTF is evaluated at $1$ and $250$ NFE from the same fine-tuned flow map.
    Results are means over three matched seeds; the base model is DMF XL/2.}
\end{table}

\begin{wrapfigure}[16]{r}{0.47\linewidth}
	\vspace{-\baselineskip}
	\centering
	\includegraphics[width=\linewidth]{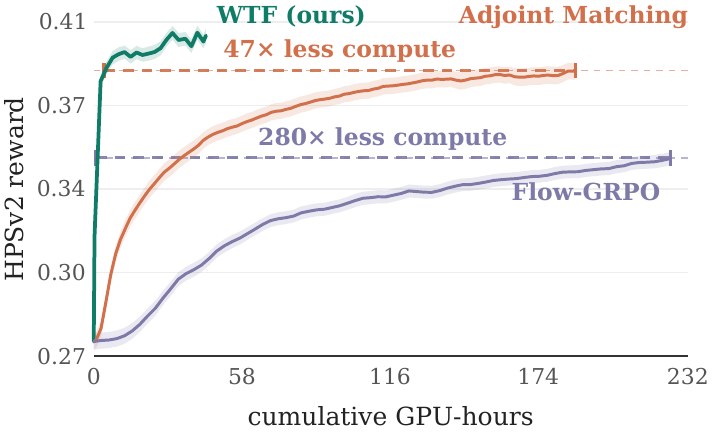}
	\caption{\textbf{Reward against cumulative training compute.}
		HPSv2 versus GPU-hours at $50$ NFE, all methods fine-tuned from the same flow map.
		WTF reaches the peak rewards of Adjoint Matching and Flow-GRPO with $47\times$ and $280\times$ less compute.
		Shading is $\pm1$ standard error; protocol in \cref{app:compute}.
		}
	\label{fig:compute_efficiency}
\end{wrapfigure}

\Cref{tab:reward-comparison-imagenet} compares WTF with Adjoint Matching, MFM, and VFM.
WTF achieves higher reward while maintaining competitive diversity.
The same checkpoint also remains effective from one to 250 NFE.
MFM fine-tunes the diagonal velocity field, while VFM targets the one-step sampler, so neither directly provides the same arbitrary-budget flow map.
\Cref{ssec:exp_imagenet_qual} shows matched-latent samples.

\FloatBarrier

\subsection{Text-to-image main results}
\label{ssec:exp_t2i}
We fine-tune the TiM-T2I flow map~\citep{wang2025transition} against HPSv2.
Every method was fine-tuned under the same budget of roughly $48$ GPU-hours on one $8 \times \mathrm{H100}$ node. \Cref{app:hyper_t2i} gives the full setup. \Cref{tab:reward-comparison-t2i} evaluates each baseline checkpoint in our harness.

\begin{table}[tb]
    \centering
    \resizebox{\textwidth}{!}{%
        \begin{tabular}{lcccccc}
            \toprule
            & & \multicolumn{3}{c}{Reward $\uparrow$} & \multicolumn{2}{c}{Diversity $\uparrow$} \\
            \cmidrule(lr){3-5} \cmidrule(lr){6-7}
            Method & NFE
                & HPSv2 (fine-tuned)
                & PickScore
                & ImageReward
                & DreamSim
                & CLIP diversity \\
            \midrule
            Base Model         & 1   &  0.228 & 20.36 &  0.182   & 0.548 & 0.413 \\
            Base Model         & 4   &  0.255 & 21.06 &  0.665   & 0.426 & 0.346 \\
            Base Model         & 8   &  0.264 & 21.29 &  0.801   & 0.361 & 0.296 \\
            Base Model         & 50  &  0.274 & 21.58 &  0.997   & 0.281 & 0.239 \\
            \midrule
            Adjoint Matching   & 50  &  0.357 & 22.75 &  1.380   & 0.210 & 0.188 \\
            Flow-GRPO           & 50  &  0.303 & 22.06 &  1.210   & 0.250 & 0.215 \\

            \midrule
            \ourrow WTF (Ours) & 1   &  0.347 & 22.31 &  1.069  & 0.325 & 0.289 \\
            \ourrow WTF (Ours) & 4   &  0.401 & 22.97 &  1.383  & 0.250 & 0.220 \\
            \ourrow WTF (Ours) & 8   &  0.400 & 22.93 &  1.394  & 0.233 & 0.212 \\
            \ourrow WTF (Ours) & 50  &  0.398 & 22.87 &  1.416  & 0.217 & 0.205 \\
            \bottomrule
        \end{tabular}%
    }
    \caption{\label{tab:reward-comparison-t2i}\textbf{Text-to-image results.}
    Baseline checkpoints are evaluated in the same harness, and WTF is evaluated across
    inference budgets from a single fine-tuned flow map.
    At 50 NFE, WTF attains the highest HPSv2 (the training reward), PickScore and ImageReward of any method in the table.}
\end{table}

\Cref{tab:reward-comparison-t2i} compares WTF with Adjoint Matching and Flow-GRPO checkpoints evaluated in the same harness.
As on ImageNet, WTF achieves higher reward while maintaining comparable diversity, and the same fine-tuned flow map remains effective in the few-step regime.
\Cref{fig:compute_efficiency} shows that WTF also converges substantially faster, reaching the peak rewards attained by Flow-GRPO and Adjoint Matching with up to $280\times$ and $47\times$ less training compute, respectively.
\Cref{ssec:exp_lambda} varies $\lambda$ to characterize the reward-diversity tradeoff, and \cref{app:compute} gives the compute comparison.

\FloatBarrier

\section{Conclusion}
\label{sec:conclusion}
We introduced a transport-regularized formulation of reward fine-tuning for deterministic flows.
The objective transports samples toward higher reward rather than reweighting the base distribution and is equivalent to a \textit{deterministic optimal control problem}.
A pre-trained flow map makes value estimation simulation-free and critic-free, yielding direct end-to-end fine-tuning of the flow map in which the same few-step sampler is used both at training and at deployment.
On ImageNet-256 and text-to-image generation, WTF matches or surpasses baselines on reward, reaching the peak rewards of the comparison methods with substantially less training compute.
The fine-tuned model remains compatible with flow map inference-time alignment without an intervening distillation stage.
We hope that our framework spurs broader interest in the use of flow maps as a foundational primitive for accelerated post-training of generative models.

\paragraph{Limitations.}
\Cref{prop:wtf_critical_point} is a population-level fixed-point result.
The large-scale experiments use the Euclidean reward gradient described in~\cref{sec:algorithmic}.
Increasing reward can reduce sample diversity, with the reward scale controlling this tradeoff.
WTF also assumes access to a pre-trained flow map. Starting from a velocity model requires a preceding flow map distillation stage, whose cost is not included in our fine-tuning comparison.

\FloatBarrier

\section*{Acknowledgements}
AM is supported by the Clarendon Fund Scholarship, University of Oxford.
We gratefully acknowledge fal for providing the computational resources that enabled this work. 
We also thank Modal for additional compute support.
The authors acknowledge the use of resources provided by the Isambard-AI National AI Research Resource (AIRR)~\citep{mcintoshsmith2024isambardaileadershipclasssupercomputer}.
Isambard-AI is operated by the University of Bristol and is funded by the UK Government's Department for Science, Innovation and Technology (DSIT) via UK Research and Innovation; and the Science and Technology Facilities Council [ST/AIRR/I-A-I/1023].

\newpage
\bibliographystyle{unsrtnat}
\bibliography{references}

\newpage

\appendix
\crefalias{section}{appendix}
\crefalias{subsection}{appendix}
\crefalias{subsubsection}{appendix}
\section{Background on flow-based generative models}
\label{app:flows_and_maps}
\label{app:si_flows}
\label{app:flow_maps}
In this section, we provide a self-contained introduction to flow matching via stochastic interpolants and their self-distillation into flow maps.

\paragraph{Stochastic interpolants and flow matching.}
Given a dataset of samples $\{x_1^i\}_{i=1}^n \sim \rho_1$ and a base distribution $\rho_0 = \Normal(0, \Id)$, the stochastic-interpolant framework~\citep{albergo2022building, albergo2023stochastic} introduces the time-dependent random variable
\begin{equation}
	\label{eqn:app_interpolant}
	I_t = \alpha_t\, x_0 + \beta_t\, x_1, \qquad x_0 \sim \rho_0,\ x_1 \sim \rho_1,
\end{equation}
with $\alpha, \beta : [0, 1] \to [0, 1]$ satisfying $\alpha_0 = \beta_1 = 1$ and $\alpha_1 = \beta_0 = 0$.
A standard result is that $\Law(I_t) = \Law(x_t)$ where $x_t$ solves the probability flow~\cref{eqn:pflow} with velocity $b_t(x) = \E[\dot{I}_t \mid I_t = x]$.
This conditional expectation is learned by minimizing the flow-matching objective~\citep{lipman2022flow, liu2022flow, albergo2023stochastic}
\begin{equation}
	\label{eqn:app_flow_matching}
	\calL_b(\hat{b}) = \E_{t, x_0, x_1}\!\left[\norm{\hat{b}_t(I_t) - \dot{I}_t}^2\right]
\end{equation}
over a class of neural networks, converting generative modeling into a regression problem.

\paragraph{Flow map parameterization.}
Throughout the paper we use the flow map parameterization
\begin{equation}
	\label{eqn:flow_map_param}
	X_{s, t}(x) = x + (t - s)\, v_{s, t}(x),
\end{equation}
where $v : [0, 1]^2 \times \R^d \to \R^d$ is the function to be learned.
On the diagonal $s = t$, $X_{t, t}$ reduces to the identity, and the \textit{tangent condition}
\begin{equation}
	\label{eqn:tangent_identity}
	v_{t, t}(x) = b_t(x)
\end{equation}
identifies the diagonal velocity of the flow map with the drift of the underlying probability flow.
We refer to $v_{t, t}$ as the \textit{implicit velocity} of $X_{s, t}$; this property allows the flow map to serve both as an accelerated sampler \emph{and} as the source of the velocity field that our fine-tuning algorithm operates on.

\paragraph{Three characterizations and self-distillation losses.}
The flow map of the deterministic dynamics $\dot{x}_t = b_t(x_t)$ admits three equivalent characterizations,
\begin{equation}
	\label{eqn:app_three_characterizations}
	\begin{aligned}
		\partial_t X_{s, t}(x)                               & = b_t(X_{s, t}(x))      &  & (\text{Lagrangian}),                    \\
		\partial_s X_{s, t}(x) + \nabla X_{s, t}(x)\, b_s(x) & = 0                     &  & (\text{Eulerian}),                      \\
		X_{s, t}(x)                                          & = X_{u, t}(X_{s, u}(x)) &  & (\text{semigroup, for } s \le u \le t),
	\end{aligned}
\end{equation}
each of which gives rise to a different self-distillation training objective by squaring the corresponding residual and replacing $b_t$ with the (stop-gradient) diagonal velocity $v_{t, t}$ via the tangent condition~\cref{eqn:tangent_identity}:
\begin{equation}
	\label{eqn:app_lsd_esd_psd}
	\begin{aligned}
		\calL_{\lsd}(X) & = \E\!\left[\, \norm{\partial_t X_{s, t}(x_s) - \sg{v_{t, t}(X_{s, t}(x_s))}}^2\, \right],              \\
		\calL_{\esd}(X) & = \E\!\left[\, \norm{\partial_s X_{s, t}(x_s) + \sg{\nabla X_{s, t}(x_s)\, v_{s, s}(x_s)}}^2\, \right], \\
		\calL_{\psd}(X) & = \E\!\left[\, \norm{X_{s, t}(x_s) - \sg{X_{u, t}(X_{s, u}(x_s))}}^2\, \right].
	\end{aligned}
\end{equation}
Each of these objectives can be augmented with the flow-matching objective~\cref{eqn:app_flow_matching} to anchor the diagonal velocity $v_{t, t}$ to the data-derived drift $b_t$.
In the main text we instead use a value gradient matching objective on the diagonal that pulls $\hat w_{t, t}$ toward the optimal control direction.

\section{Background on the stochastic optimal control formulation of fine-tuning}
\label{app:soc_background}
In this section, we provide the standard stochastic optimal control formulation of fine-tuning, following the convention of~\citet{domingoenrich2025adjointmatchingfinetuningflow}, with notation aligned to ours.

\paragraph{Setup.}
Let $b_t : \R^d \to \R^d$ be the pre-trained probability-flow drift, let $\rho_t$ denote its marginals, and let $r : \R^d \to \R$ be a terminal reward.
For a diffusion scale $\sigma_t$, define the score-corrected stochastic drift
\begin{equation}
	\label{eqn:soc_score_corrected_drift}
	a_t(x) = b_t(x) + \tfrac{1}{2}\, \sigma_t^2\, \nabla \log \rho_t(x).
\end{equation}
The reference process
\begin{equation}
	\label{eqn:soc_reference_sde}
	dX_t = a_t(X_t)\, dt + \sigma_t\, dW_t, \qquad X_0 \sim \rho_0,
\end{equation}
has the same one-time marginals $\rho_t$ as the deterministic probability flow, with $\sigma_t > 0$ a time-dependent diffusion coefficient and $W_t$ a standard Brownian motion.
A controlled process $X^u_t$ evolves according to
\begin{equation}
	\label{eqn:soc_controlled_sde}
	dX^u_t = \big(a_t(X^u_t) + \sigma_t\, u_t(X^u_t)\big)\, dt + \sigma_t\, dW_t, \qquad X^u_0 \sim \rho_0,
\end{equation}
for a square-integrable feedback control $u : [0, 1] \times \R^d \to \R^d$.
Girsanov's theorem says that the path-space KL between $P^u$ and the reference $R$ is given by~\citep{karatzas2014brownian}
\begin{equation}
	\label{eqn:girsanov}
	\kl{P^u}{R} = \frac{1}{2} \E_{P^u}\!\left[\int_0^1 \norm{u_t(X^u_t)}^2\, dt\right].
\end{equation}
The reward-tilted distribution
\begin{equation}
	\label{eqn:soc_tilt}
	\tilde \rho_1(x) \propto e^{\lambda\, r(x)}\, \rho_1(x)
\end{equation}
arises as the maximizer of $\lambda\, \E[r(X_1)] - \kl{\tilde \rho_1}{\rho_1}$, where $\rho_1 = \Law(X_1)$ under the reference~\cref{eqn:soc_reference_sde}.

\paragraph{Lifting the reward-tilted problem to path space.}
Combining the reward with the path-space KL~\cref{eqn:girsanov} gives the regularized reward-maximization problem
\begin{equation}
	\label{eqn:soc_problem}
	\begin{aligned}
		 & \sup_u\ \E_{P^u}\!\left[\, \lambda\, r(X^u_1) - \tfrac{1}{2} \int_0^1 \norm{u_t(X^u_t)}^2\, dt\, \right],              \\
		 & \text{subject to} \quad dX^u_t = \big(a_t(X^u_t) + \sigma_t\, u_t(X^u_t)\big)\, dt + \sigma_t\, dW_t, \quad X^u_0 \sim \rho_0,
	\end{aligned}
\end{equation}
a stochastic optimal control problem.
The value function under a candidate control $u$ measures the expected cost-to-go from state $x$ at time $t$,
\begin{equation}
	\label{eqn:soc_value_general}
	V^u_t(x) = \E_{P^u}\!\left[\, \lambda\, r(X^u_1) - \tfrac{1}{2} \int_t^1 \norm{u_s(X^u_s)}^2\, ds \;\Big|\; X^u_t = x\, \right],
\end{equation}
and the optimal value function $V^*_t(x) = \sup_u V^u_t(x)$ admits the closed-form Cole--Hopf representation
\begin{equation}
	\label{eqn:soc_value}
	V^*_t(x) = \log \E_R\!\left[\, e^{\lambda\, r(X_1)} \mid X_t = x\, \right], \qquad u^*_t(x) = \sigma_t\, \nabla_x V^*_t(x),
\end{equation}
where the expectation is taken under the reference process~\cref{eqn:soc_reference_sde} rather than under $P^*$.
The optimally-controlled process $P^*$ is the Doob $h$-transform of $R$ with $h_t(x) = e^{V^*_t(x)} = \E_R[e^{\lambda\, r(X_1)} \mid X_t = x]$, whose joint density on the endpoints is
\begin{equation}
	\label{eqn:soc_joint}
	P^*(x_0, x_1) = \rho_0(x_0)\, \rho_{1 \mid 0}(x_1 \mid x_0)\, \frac{e^{\lambda\, r(x_1)}}{h_0(x_0)},
\end{equation}
where $\rho_{1 \mid 0}(x_1 \mid x_0)$ is the conditional density of $X_1$ given $X_0 = x_0$ under the reference process.
In general, the implicit terminal marginal obtained by integrating~\cref{eqn:soc_joint} over $x_0$ is not the target tilt $\tilde \rho_1$, because the $1/h_0(x_0)$ factor couples the endpoints through the initial-time value function.
\citet{domingoenrich2025adjointmatchingfinetuningflow} identify the so-called \emph{memoryless} noise schedule
\begin{equation}
	\label{eqn:memoryless_schedule}
	\sigma^{\mathrm{ml}}_t = \sqrt{2\, \eta_t}, \qquad \eta_t = \alpha_t\!\left(\frac{\dot{\beta}_t}{\beta_t}\, \alpha_t - \dot{\alpha}_t\right),
\end{equation}
which for the standard linear interpolant $\alpha_t = 1 - t,\ \beta_t = t$ reduces to $\sigma^{\mathrm{ml}}_t = \sqrt{2(1 - t)/t}$.
This schedule ensures that $X_0$ is independent of $X_1$ under the reference process, so that $\rho_{1 \mid 0}(x_1 \mid x_0) = \rho_1(x_1)$ and~\cref{eqn:soc_joint} factors and the terminal marginal is precisely the reward tilt, $P^*(x_1) \propto e^{\lambda\, r(x_1)}\, \rho_1(x_1) = \tilde \rho_1(x_1)$.

In~\cref{app:sb_zero_noise} we connect this SOC recipe to the abstract terminal-distribution framing of~\cref{eqn:wtf-abstract}, where we show that the implicit regularizer behind the SOC problem is the entropic Schr\"odinger-bridge cost between $\rho_0$ and the candidate terminal $\nu$ (\cref{prop:pathkl-to-sb}).
This cost has our static optimal transport regularizer $\Tprior$ as its zero-noise limit (\cref{prop:sb-zero-noise}).

\section{Background on deterministic optimal control}
\label{app:doc_background}
In this section, we collect several standard results from deterministic optimal control theory that we use to design our algorithm in~\cref{sec:training}.
For textbook treatments, we recommend~\citet{fleming1975deterministic,bardi1997optimal,liberzon2011calculus}.

\paragraph{Setup.}
We study the deterministic optimal control problem~\cref{eqn:oc_problem},
\begin{equation}
	\label{eqn:doc_problem}
	\begin{aligned}
		 & \sup_u\ \E_{x_0 \sim \rho_0}\!\left[\, \lambda\, r(x_1^u) - \frac{1}{2} \int_0^1 \norm{u_t(x_t^u)}^2\, dt\, \right], \\
		 & \text{subject to} \quad \dot{x}_t^u = b_t(x_t^u) + u_t(x_t^u), \quad x_0^u = x_0,
	\end{aligned}
\end{equation}
through a feedback control $u : [0, 1] \times \R^d \to \R^d$.
The value function under a candidate control $u$~\cref{eqn:value_fn},
\begin{equation}
	\label{eqn:doc_value}
	V^u_t(x) = \lambda\, r(X^u_{t, 1}(x)) - \frac{1}{2} \int_t^1 \norm{u_\tau(X^u_{t, \tau}(x))}^2\, d\tau,
\end{equation}
measures the reward-to-go from state $x$ at time $t$, and the optimal value function is $V^*_t(x) = \sup_u V^u_t(x)$.
This definition mirrors its stochastic counterpart~\cref{eqn:soc_value_general}, but in the deterministic setting $V^*$ admits no Cole--Hopf representation as an expectation over the reference process: the deterministic reference flow has no randomness to integrate against, and the closed form~\cref{eqn:soc_value} collapses to a pointwise evaluation that returns no information about the optimal control.
We instead characterize $V^*$ via the Hamilton--Jacobi--Bellman equation, derived next.

\paragraph{Hamilton--Jacobi--Bellman equation.}
A standard dynamic-programming argument shows that the optimal value function $V^*$ solves the Hamilton--Jacobi--Bellman equation
\begin{equation}
	\label{eqn:doc_hjb}
	\begin{aligned}
		\partial_t V^*_t(x) + \sup_{u \in \R^d}\!\left\{\, (b_t(x) + u) \cdot \nabla_x V^*_t(x) - \tfrac{1}{2} \norm{u}^2\, \right\} = 0,\quad
		V^*_1(x) = \lambda\, r(x).
	\end{aligned}
\end{equation}
The pointwise supremum is attained at the optimal control
\begin{equation}
	\label{eqn:doc_optimal_control}
	u^*_t(x) = \nabla_x V^*_t(x),
\end{equation}
and substituting back into~\cref{eqn:doc_hjb} reduces to
\begin{equation}
	\label{eqn:doc_hjb_reduced}
	\partial_t V^*_t(x) + b_t(x) \cdot \nabla_x V^*_t(x) + \tfrac{1}{2} \norm{\nabla_x V^*_t(x)}^2 = 0.
\end{equation}
The closed-form solution~\cref{eqn:doc_optimal_control} is what motivates the value gradient regression target in~\cref{sec:training}: knowing $V^*$ pointwise gives the optimal control directly via differentiation in $x$.

\paragraph{Policy-evaluation identity.}
For any fixed reference control $\bar u$ (not necessarily optimal), the value function $V^{\bar u}$ satisfies a linear transport equation that we record as a lemma for later reuse.
This identity is the workhorse behind the proof of the performance-difference proposition~\cref{prop:performance_diff} given in~\cref{app:performance_diff}.

\begin{lemma}[Policy evaluation]
	\label{lem:policy_evaluation}
	Under the regularity assumptions of~\cref{app:standing_assumptions}, the value function $V^{\bar u}$ of any feedback control $\bar u \in \calU$ satisfies the linear transport equation
	\begin{equation}
		\label{eqn:doc_policy_eval}
		\partial_t V^{\bar u}_t(x) + (b_t(x) + \bar u_t(x)) \cdot \nabla_x V^{\bar u}_t(x) = \tfrac{1}{2} \norm{\bar u_t(x)}^2.
	\end{equation}
\end{lemma}
\begin{proof}
	Fix $\bar u \in \calU$ and let $x_t^{\bar u}$ denote the controlled trajectory $\dot x_t^{\bar u} = b_t(x_t^{\bar u}) + \bar u_t(x_t^{\bar u})$.
	From the integral definition~\cref{eqn:doc_value} of $V^{\bar u}$, the value along the controlled trajectory satisfies
	\begin{equation}
		\label{eqn:policy_eval_proof_integral}
		V^{\bar u}_t(x_t^{\bar u}) = \lambda\, r(x_1^{\bar u}) - \int_t^1 \tfrac{1}{2} \norm{\bar u_\tau(x_\tau^{\bar u})}^2\, d\tau,
	\end{equation}
	whose right-hand side depends on $t$ only through the lower limit of the integral.
	Differentiating in $t$ gives the trajectory-level identity
	\begin{equation}
		\label{eqn:policy_eval_proof_lagrangian}
		\frac{d}{dt} V^{\bar u}_t(x_t^{\bar u}) = \tfrac{1}{2} \norm{\bar u_t(x_t^{\bar u})}^2.
	\end{equation}
	By the chain rule, the same derivative equals
	\begin{equation}
		\label{eqn:policy_eval_proof_chainrule}
		\frac{d}{dt} V^{\bar u}_t(x_t^{\bar u}) = \partial_t V^{\bar u}_t(x_t^{\bar u}) + \dot x_t^{\bar u} \cdot \nabla_x V^{\bar u}_t(x_t^{\bar u}) = \partial_t V^{\bar u}_t(x_t^{\bar u}) + (b_t + \bar u_t)(x_t^{\bar u}) \cdot \nabla_x V^{\bar u}_t(x_t^{\bar u}).
	\end{equation}
	Equating~\cref{eqn:policy_eval_proof_lagrangian} and~\cref{eqn:policy_eval_proof_chainrule} and noting that varying the initial condition $x_0$ makes $x_t^{\bar u}$ range over $\R^d$ at each $t$ yields~\cref{eqn:doc_policy_eval} pointwise.
\end{proof}

At the optimal control $\bar u = u^* = \nabla_x V^*$, comparing~\cref{eqn:doc_policy_eval} with the reduced HJB~\cref{eqn:doc_hjb_reduced} confirms the consistency $V^{u^*} = V^*$.

\section{Additional derivations}
\label{app:derivations}
In this section, we collect additional derivations and structural results that support the development in the main text.

\subsection{Regularity assumptions}
\label{app:standing_assumptions}
In the following, we assume the below standard regularity assumptions.

\begin{assumption}[Drift regularity]
	\label{assump:drift}
	The pre-trained drift $b : [0, 1] \times \R^d \to \R^d$ is jointly measurable, locally Lipschitz in $x$ uniformly in $t$, and has at most linear growth at infinity, so the ODE $\dot{x}_t = b_t(x_t)$ generates a global flow $X:[0, 1]^2\times\R^d\to\R^d$ on $\R^d$.
\end{assumption}

\vspace{0.6em}
\begin{assumption}[Reward regularity]
	\label{assump:reward}
	The reward $r : \R^d \to \R$ is bounded above and upper semicontinuous; for results that involve $\nabla r$ we further assume $r \in C^1(\R^d)$ with bounded gradient.
\end{assumption}

\vspace{0.6em}
\begin{assumption}[Marginals]
	\label{assump:marginals}
	The base distribution $\rho_0$ and the candidate terminal distributions $\nu$ have finite second moment, with $\Tprior(\rho_0, \nu) < \infty$ for the relevant $\nu$.
\end{assumption}

\begin{assumption}[Trainable flow map class]
	\label{assump:flowmap_class}
	The trainable field $\hat w$ ranges over a convex class $\calW$ of two-time velocity fields $\hat w : \{0 \le s \le t \le 1\} \times \R^d \to \R^d$ that are continuously differentiable in $x$ and in $s$, and locally Lipschitz in $x$ uniformly in $(s, t)$.
	We assume $\calW$ contains, for every continuous field $\phi$ on the upper triangle, the averaged field $(s, t, x) \mapsto \tfrac{1}{t - s}\int_s^t \phi(\sigma, t, x)\, d\sigma$, and that $\hat w_{t, t} - b_t$ is an admissible control for every $\hat w \in \calW$.
\end{assumption}

\begin{assumption}[Control class]
	\label{assump:control}
	The admissible control class $\calU$ consists of locally Lipschitz feedback controls $u : [0, 1] \times \R^d \to \R^d$ satisfying $\E_{x_0 \sim \rho_0}\!\left[\int_0^1 \norm{u_t(x_t^u)}^2 dt\right] < \infty$.
\end{assumption}

These assumptions are standard in the theory of dynamical systems and optimal control to ensure uniqueness of solutions~\citep{fleming1975deterministic,bardi1997optimal,liberzon2011calculus,evans2010partial}, and in optimal transport theory to ensure existence of solutions to the corresponding optimization problems~\citep{villani2009optimal,santambrogio2015optimal}.
In particular,~\cref{assump:drift,assump:reward,assump:marginals} together imply that the cost $\cprior$ is lower semicontinuous and coercive in its endpoints, so that the static infimum $\Tprior(\rho_0, \nu)$ is attained.

\subsection{Schr\"odinger bridges and the zero-noise limit}
\label{app:sb_zero_noise}
The SOC approach in~\cref{app:soc_background} optimizes over path measures $P$, but the reward depends only on the terminal $X_1 \sim P_1$.
We first show that the implicit terminal-level regularizer behind this recipe, in the sense of~\cref{eqn:wtf-abstract}, is the entropic Schr\"odinger-bridge cost (\cref{prop:pathkl-to-sb}), then take the zero-noise limit of this cost to recover our static prior-action transport regularizer $\Tprior$ (\cref{prop:sb-zero-noise}).

For a noise scale $\varepsilon > 0$, let $R^\varepsilon$ denote the reference path measure on $C([0, 1]; \R^d)$ associated with the reference SDE
\begin{equation}
	\label{eqn:sb_reference_sde}
	dX_t = a_t(X_t)\, dt + \sqrt{\varepsilon}\, dW_t, \qquad X_0 \sim \rho_0,
\end{equation}
which is the constant-noise specialization of~\cref{app:soc_background} with $\sigma_t = \sqrt{\varepsilon}$. It preserves the marginals $\rho_t$ and converges to the deterministic probability flow as $\varepsilon \to 0$.
The entropic Schr\"odinger-bridge cost between $\rho_0$ and a candidate terminal $\nu$ is the infimum of the path-space KL over all path measures with prescribed marginals,
\begin{equation}
	\label{eqn:SB-cost}
	\SBeps(\rho_0, \nu) = \inf_{P\,:\, P_0 = \rho_0,\, P_1 = \nu}\, \kl{P}{R^\varepsilon}.
\end{equation}

\begin{proposition}[Path-KL collapses to entropic SB at the terminal level]
	\label{prop:pathkl-to-sb}
	For any $\lambda > 0$ and $\varepsilon > 0$, the path-space fine-tuning problem
	\begin{equation}
		\label{eqn:soc_pathspace}
		\sup_{P\,:\, P_0 = \rho_0}\ \left\{\ \lambda\, \E_{P}[r(X_1)] - \kl{P}{R^\varepsilon}\ \right\}
	\end{equation}
	has the same value as the terminal-space problem
	\begin{equation}
		\label{eqn:soc_terminal}
		\sup_\nu\ \left\{\ \lambda\, \E_{x \sim \nu}[r(x)] - \SBeps(\rho_0, \nu)\ \right\}.
	\end{equation}
	A path measure $P^*$ achieves the supremum in~\cref{eqn:soc_pathspace} if and only if its terminal law $\nu^* = P^*_1$ achieves the supremum in~\cref{eqn:soc_terminal} and $P^*$ achieves the infimum~\cref{eqn:SB-cost} defining $\SBeps(\rho_0, \nu^*)$.
\end{proposition}

\begin{proof}
	The reward $\E_P[r(X_1)] = \E_{x \sim P_1}[r(x)]$ depends on $P$ only through its terminal marginal.
	Splitting the supremum,
	\begin{align}
		\sup_{P\,:\, P_0 = \rho_0}\ \big\{ \lambda\, \E_P[r(X_1)] - \kl{P}{R^\varepsilon} \big\}
		 & = \sup_\nu\ \sup_{P\,:\, P_0 = \rho_0,\, P_1 = \nu}\ \big\{ \lambda\, \E_\nu[r] - \kl{P}{R^\varepsilon} \big\} \\
		 & = \sup_\nu\ \big\{ \lambda\, \E_\nu[r] - \inf_{P\,:\, P_0 = \rho_0,\, P_1 = \nu} \kl{P}{R^\varepsilon} \big\},
	\end{align}
	where the second equality moves the $\inf$ inside because the reward term does not depend on $P$.
	The inner infimum is $\SBeps(\rho_0, \nu)$ by definition, giving the value equality.
	The optimizer characterization follows from the same splitting, as a maximizer $P^*$ of the joint supremum must achieve both the outer supremum (in $\nu^*$) and the inner infimum (for that $\nu^*$), and conversely.
\end{proof}

The following result clarifies the relationship between Adjoint Matching and the WTF formulation of~\cref{sec:methods}, showing that WTF can be viewed as a suitable zero-noise limit of a Schr\"odinger bridge problem, while Adjoint Matching is a specific Schr\"odinger bridge problem with the memoryless noise schedule.
\begin{proposition}[Zero-noise limit]
	\label{prop:sb-zero-noise}
	Under the standing regularity assumptions of~\cref{app:standing_assumptions}, the rescaled entropic Schr\"odinger-bridge cost converges to the static prior-action transport cost,
	\begin{equation}
		\label{eqn:sb_zero_noise_main}
		\varepsilon\, \SBeps(\rho_0, \nu) \;\longrightarrow\; \Tprior(\rho_0, \nu)
		\quad \text{as } \varepsilon \to 0^+,
	\end{equation}
	and the corresponding stochastic optimal control formulation of $\SBeps$ reduces in the same limit to the deterministic OC problem~\cref{eqn:oc_problem}.
\end{proposition}

The proof combines Freidlin--Wentzell large-deviations and $\Gamma$-convergence with the dynamic-form representation of $\Tprior$ from~\cref{lem:dynamic_static}.
The key steps are carried out by~\citet{chen2016SBcontrol,chen2021stochastic}; we refer the reader there for the proof.
\Cref{prop:sb-zero-noise} formalizes the statement that WTF is the deterministic limit of the path-KL fine-tuning recipe of~\cref{app:soc_background}.
\Cref{tab:regularizer-comparison} situates WTF alongside the KL-based and SB-based recipes by what each regularizer constrains, what algorithm it induces, and the rollout cost it pays at training time.

\begin{table}[t]
	\centering
	\footnotesize
	\setlength{\tabcolsep}{4pt}
	\resizebox{\textwidth}{!}{%
		\begin{tabular}{@{}llllll@{}}
			\toprule
			Regularizer            & Object                 & Optimizer                                      & Interpretation                  & Training rollout       \\
			\midrule
			$\kl{\nu}{\rho_1}$     & terminal law           & reward tilt $\propto e^{\lambda r} \rho_1$     & exact reward tilt / posterior   & multi-step             \\
			$\SBeps(\rho_0, \nu)$  & terminal law via SB    & Doob $h$-transform of $R$                      & stochastic-process trust region & multi-step (SDE)       \\
			$\Tprior(\rho_0, \nu)$ & terminal law via $b_t$ & deterministic OC (\cref{thm:wasserstein_tilt}) & minimum intervention in flow    & $O(1)$ flow map evaluations \\
			\bottomrule
		\end{tabular}%
	}
	\vspace{0.5em}
	\caption{\textbf{Comparison of methods.}
		Regularizers for fine-tuning generative models, organized by what they constrain. WTF (last row) is defined directly in terms of the deterministic dynamics of the pre-trained flow, which lets the resulting algorithm exploit the flow map for constant-number flow map evaluations at training time.}
	\label{tab:regularizer-comparison}
\end{table}

\subsection{Failure of the naive control objective}
\label{app:naive_fail}
The most direct algorithmic translation of the optimal control problem~\cref{eqn:oc_problem} uses the fine-tuned flow map~\cref{eqn:w-flow-map} $X^{\hat w}_{s, t}(x) = x + (t - s)\, \hat w_{s, t}(x)$ from the main text and writes the residual control as $u_t = \hat w_{t, t} - b_t$, leading to the sampled objective
\begin{equation}
	\label{eqn:naive_appendix}
	\calL(\hat w) = \E_{x_0, t}\!\left[\, \tfrac12 \norm{\hat w_{t, t}(X^{\hat w}_{0, t}(x_0)) - b_t(X^{\hat w}_{0, t}(x_0))}_2^2 - \lambda\, r(X^{\hat w}_{0, 1}(x_0))\, \right] + \beta\, \calL_{\dist}(\hat w),
\end{equation}
where $\calL_\dist$ is an off-diagonal self-distillation loss.
In practice we find that~\cref{eqn:naive_appendix} fails to provide a sufficiently informative learning signal for the diagonal velocity $\hat w_{t, t}$.
The reward acts on the terminal map $X^{\hat w}_{0, 1}$, which depends on the off-diagonal $\hat w_{s, t}$ for $(s, t)$ all the way up to $(0, 1)$.
The control cost in contrast only sees $\hat w_{t, t}$ on the diagonal.
The model can therefore change the terminal map in directions that increase reward while leaving the sampled diagonal velocity nearly unchanged.
This creates a map-dynamics inconsistency: the off-diagonal map adapts without learning the corresponding diagonal dynamics.
Early experiments with this objective motivated the value-gradient approach in the main text.

\section{Omitted proofs}
\label{app:proofs}
In this section, we restate and prove the mathematical results from the main text.

\subsection{\texorpdfstring{Proof of \cref{prop:pointwise-decomp}}{Proof of the per-particle decomposition}}
\label{app:pointwise_decomp}
\ppdecomp*
\begin{proof}
	Substituting the Kantorovich definition~\cref{eqn:Tb-static} of $\Tprior$ into the static reward-regularized problem,
	\begin{equation}
		\label{eqn:pp_proof_kant_substitution}
		\begin{aligned}
			\sup_\nu \big\{ \lambda\, \E_{y \sim \nu}[r(y)] - \Tprior(\rho_0, \nu) \big\}
			 & = \sup_\nu\, \sup_{\pi \in \Pi(\rho_0,\, \nu)}\ \int \big[\, \lambda\, r(y) - \cprior(x, y)\, \big]\, \pi(dx\, dy) \\
			 & = \sup_{\pi\,:\, \pi_0 = \rho_0}\ \int \big[\, \lambda\, r(y) - \cprior(x, y)\, \big]\, \pi(dx\, dy),
		\end{aligned}
	\end{equation}
	where the second equality uses that the joint sup over $(\nu, \pi)$ with $\pi \in \Pi(\rho_0, \nu)$ reduces to a sup over couplings $\pi$ with first marginal $\rho_0$ and arbitrary second marginal.
	The constraint $\pi_0 = \rho_0$ lets us disintegrate any feasible coupling as $\pi(dx\, dy) = \rho_0(dx)\, \pi_{1 \mid 0}(dy \mid x)$, where for each $x$ the conditional $\pi_{1 \mid 0}(\cdot \mid x)$ is a probability measure on $\R^d$ and is otherwise unconstrained.
	Substituting this disintegration gives
	\begin{equation}
		\label{eqn:pp_proof_disintegration}
		\begin{aligned}
			 & \sup_{\pi\,:\, \pi_0 = \rho_0}\ \int \big[\lambda\, r(y) - \cprior(x, y)\big]\, \pi(dx\, dy)                                             \\
			 & = \sup_{\{\pi_{1 \mid 0}(\cdot \mid x)\}_x}\ \int \rho_0(dx)\, \int \big[\lambda\, r(y) - \cprior(x, y)\big]\, \pi_{1 \mid 0}(dy \mid x) \\
			 & = \int \rho_0(dx)\, \sup_{\pi_{1 \mid 0}(\cdot \mid x)}\, \int \big[\lambda\, r(y) - \cprior(x, y)\big]\, \pi_{1 \mid 0}(dy \mid x)      \\
			 & = \E_{x \sim \rho_0}\!\left[\, \sup_{y \in \R^d}\ \big\{\, \lambda\, r(y) - \cprior(x, y)\, \big\}\, \right],
		\end{aligned}
	\end{equation}
	which is~\cref{eqn:pointwise-decomp}.
	The first equality substitutes the disintegration, the second exchanges the sup with the outer integral, and the third observes that the inner sup over probability measures on $\R^d$ of the linear functional $\mu \mapsto \int f(y)\, \mu(dy)$ with $f(y) = \lambda\, r(y) - \cprior(x, y)$ equals $\sup_y f(y)$, attained by a Dirac at any maximizer $y^*(x) \in \argmax_y \{\lambda\, r(y) - \cprior(x, y)\}$.
\end{proof}

\subsection{\texorpdfstring{Proof of \cref{thm:wasserstein_tilt}}{Proof of the static-OC equivalence theorem}}
\label{app:wasserstein_tilt}
The proof relies on the dynamic formulation of optimal transport with prior dynamics established by~\citet{chen2017OTlineardynamics,chen2016SBcontrol,chen2021stochastic}, due to~\citet{chen2017OTlineardynamics} for linear priors and to~\citet{chen2016SBcontrol,chen2021stochastic} for general drifts.
We state the result in our notation for completeness; it holds under the regularity assumptions of~\cref{app:standing_assumptions}, and we refer the reader to those works for its proof.

\begin{lemma}[Benamou--Brenier with prior dynamics]
	\label{lem:dynamic_static}
	Under the regularity assumptions of~\cref{app:standing_assumptions}, for every $\nu$ such that $\Tprior(\rho_0, \nu) < \infty$,
	\begin{equation}
		\label{eqn:dynamic_static_lemma}
		\begin{aligned}
			\Tprior(\rho_0, \nu) = \inf_{(\rho, v)} \quad & \, \int_0^1 \!\!\int \tfrac{1}{2}\norm{v_t(x) - b_t(x)}^2\, \rho_t(x)\, dx\, dt,                                 \\
			\text{subject to} \quad                       & \, \partial_t \rho_t + \nabla \!\cdot\! (v_t\, \rho_t) = 0, \quad \rho|_{t=0} = \rho_0, \quad \rho|_{t=1} = \nu.
		\end{aligned}
	\end{equation}
\end{lemma}

We now use~\cref{lem:dynamic_static} to prove~\cref{thm:wasserstein_tilt}.
\wassersteintheorem*
\begin{proof}
	For brevity we write
	\begin{equation}
		\label{eqn:woc_proof_objectives}
		\begin{aligned}
			J_{\mathrm{OC}}(u)     & = \E_{x_0 \sim \rho_0}\!\left[\lambda\, r(x_1^u) - \tfrac{1}{2}\int_0^1 \norm{u_t(x_t^u)}^2\, dt\right], \\
			J_{\mathrm{stat}}(\nu) & = \lambda\, \E_{x \sim \nu}[r(x)] - \Tprior(\rho_0, \nu),
		\end{aligned}
	\end{equation}
	for the OC and static objectives, respectively.
	We prove the value equality
	\begin{equation}
		\label{eqn:woc_proof_value_equality}
		\sup_{u \in \calU} J_{\mathrm{OC}}(u) = \sup_\nu J_{\mathrm{stat}}(\nu)
	\end{equation}
	via matching upper and lower bounds.
	The lower-bound construction also yields the optimizer-realization claim under the regularity hypotheses.

	We first prove the upper bound $\sup_u J_{\mathrm{OC}}(u) \le \sup_\nu J_{\mathrm{stat}}(\nu)$.
	Fix any $u \in \calU$ and let $\rho_t^u = \Law(x_t^u)$ denote the law of the controlled trajectory, with corresponding velocity $v_t^u = b_t + u_t$.
	By the standard correspondence between trajectories and densities, $(\rho^u, v^u)$ satisfies the continuity equation
	\begin{equation}
		\label{eqn:woc_proof_continuity}
		\partial_t \rho_t^u + \nabla \cdot (v_t^u\, \rho_t^u) = 0,
		\qquad \rho^u|_{t=0} = \rho_0,
		\qquad \rho^u|_{t=1} = \Law(x_1^u),
	\end{equation}
	so $(\rho^u, v^u)$ is a feasible candidate for the dynamic prior-action transport problem~\cref{eqn:dynamic_static_lemma} from $\rho_0$ to $\Law(x_1^u)$.
	Pulling the squared deviation $\tfrac{1}{2}\norm{v_t^u - b_t}^2 = \tfrac{1}{2}\norm{u_t}^2$ back to trajectories under $\rho_t^u$ gives the action identity
	\begin{equation}
		\label{eqn:woc_proof_cost_identity}
		\int_0^1 \!\!\int \tfrac{1}{2}\norm{v_t^u(x) - b_t(x)}^2\, \rho_t^u(x)\, dx\, dt
		= \tfrac{1}{2}\, \E_{x_0 \sim \rho_0}\!\left[\int_0^1 \norm{u_t(x_t^u)}^2\, dt\right],
	\end{equation}
	which is the change of variables from densities to trajectories, using $\Law(x^u_t) = \rho^u_t$.
	Applying~\cref{lem:dynamic_static} to the candidate $(\rho^u, v^u)$, the right-hand side is at least $\Tprior(\rho_0, \Law(x_1^u))$, and so
	\begin{equation}
		\label{eqn:woc_proof_upper_chain}
		\begin{aligned}
			J_{\mathrm{OC}}(u)
			\; & =\; \lambda\, \E_{x \sim \Law(x_1^u)}[r(x)] - \tfrac{1}{2}\, \E_{x_0 \sim \rho_0}\!\left[\int_0^1 \norm{u_t(x_t^u)}^2\, dt\right] \\
			\; & \le\; \lambda\, \E_{x \sim \Law(x_1^u)}[r(x)] - \Tprior(\rho_0, \Law(x_1^u))
			\;=\; J_{\mathrm{stat}}(\Law(x_1^u)).
		\end{aligned}
	\end{equation}
	The first equality is the change of variables $\E_{x_0 \sim \rho_0}[r(x_1^u(x_0))] = \E_{x \sim \Law(x_1^u)}[r(x)]$, the inequality is~\cref{lem:dynamic_static} applied to $(\rho^u, v^u)$ via~\cref{eqn:woc_proof_cost_identity}, and the final equality is the definition of $J_{\mathrm{stat}}$.
	Taking the sup over $u \in \calU$ on both sides and using that the set of attainable terminal laws $\{\Law(x_1^u) : u \in \calU\}$ is contained in the admissible set for $\sup_\nu J_{\mathrm{stat}}(\nu)$ yields $\sup_u J_{\mathrm{OC}}(u) \le \sup_\nu J_{\mathrm{stat}}(\nu)$.

	We now prove the matching lower bound $\sup_u J_{\mathrm{OC}}(u) \ge \sup_\nu J_{\mathrm{stat}}(\nu)$.
	Let $\nu^*$ achieve the supremum of $J_{\mathrm{stat}}$, whose existence follows from the assumptions of~\cref{app:standing_assumptions}.
	By~\cref{lem:dynamic_static}, there exists a pair $(\rho^*, v^*)$ achieving the dynamic prior-action cost,
	\begin{equation}
		\label{eqn:woc_proof_dynamic_optimizer}
		\int_0^1 \!\!\int \tfrac{1}{2}\norm{v_t^*(x) - b_t(x)}^2\, \rho_t^*(x)\, dx\, dt = \Tprior(\rho_0, \nu^*),
		\qquad \rho^*|_{t=0} = \rho_0, \qquad \rho^*|_{t=1} = \nu^*.
	\end{equation}
	Set $u^*_t = v^*_t - b_t$.
	Under the regularity assumed in the theorem statement, $u^* \in \calU$ and the controlled flow $\dot x^{u^*}_t = b_t(x^{u^*}_t) + u^*_t(x^{u^*}_t)$ generates the velocity field $v^*$, so that $\Law(x_t^{u^*}) = \rho_t^*$ and in particular $\Law(x_1^{u^*}) = \nu^*$.
	Pulling the cost back to trajectories as in~\cref{eqn:woc_proof_cost_identity} and substituting~\cref{eqn:woc_proof_dynamic_optimizer} yields
	\begin{equation}
		\label{eqn:woc_proof_lower_chain}
		\begin{aligned}
			J_{\mathrm{OC}}(u^*)
			\; & =\; \lambda\, \E_{x \sim \nu^*}[r(x)] - \tfrac{1}{2}\, \E_{x_0 \sim \rho_0}\!\left[\int_0^1 \norm{u^*_t(x_t^{u^*})}^2\, dt\right] \\
			\; & =\; \lambda\, \E_{x \sim \nu^*}[r(x)] - \Tprior(\rho_0, \nu^*) \;=\; J_{\mathrm{stat}}(\nu^*),
		\end{aligned}
	\end{equation}
	where the first equality uses $\Law(x_1^{u^*}) = \nu^*$, the second substitutes~\cref{eqn:woc_proof_dynamic_optimizer} via~\cref{eqn:woc_proof_cost_identity}, and the third is the definition of $J_{\mathrm{stat}}$.
	Since $u^*$ is feasible, $\sup_u J_{\mathrm{OC}}(u) \ge J_{\mathrm{OC}}(u^*) = J_{\mathrm{stat}}(\nu^*) = \sup_\nu J_{\mathrm{stat}}(\nu)$.

	Combining the upper and lower bounds gives the value equality~\cref{eqn:woc_proof_value_equality}.
	The lower-bound construction additionally yields the optimizer-realization claim: the specific feasible control $u^* = v^* - b$ achieves $J_{\mathrm{OC}}(u^*) = J_{\mathrm{stat}}(\nu^*)$ by~\cref{eqn:woc_proof_lower_chain}, this in turn equals $\sup_\nu J_{\mathrm{stat}}(\nu)$ by the choice of $\nu^*$, and the value equality~\cref{eqn:woc_proof_value_equality} identifies the right-hand side with $\sup_u J_{\mathrm{OC}}(u)$.
	Hence $J_{\mathrm{OC}}(u^*) = \sup_u J_{\mathrm{OC}}(u)$ with $u^* \in \calU$, so $u^*$ attains the OC supremum, and its terminal law $\Law(x_1^{u^*}) = \nu^*$ solves~\cref{eqn:wtf-reward-objective}.
\end{proof}

\subsection{\texorpdfstring{Proof of \cref{prop:performance_diff}}{Proof of the policy-improvement proposition}}
\label{app:performance_diff}
\performancediff*
\begin{proof}
	We first establish the performance-difference identity~\cref{eqn:performance_diff} and then specialize to $w = \bar g$.
	Fix a candidate control $w$ and an initial pair $(s, x)$, and let $x^w_t = X^w_{s, t}(x)$ denote the controlled trajectory of $\dot x^w_t = b_t(x^w_t) + w_t(x^w_t)$ for $t \in [s, 1]$.

	Differentiating $V^{\bar u}_t(x^w_t)$ in $t$ along the trajectory, applying~\cref{lem:policy_evaluation} at $(t, x^w_t)$, and using $\bar g_t(x^w_t) = \nabla_x V^{\bar u}_t(x^w_t)$ yields
	\begin{equation}
		\label{eqn:perf_proof_chainrule}
		\frac{d}{dt} V^{\bar u}_t(x^w_t) = \tfrac{1}{2} \norm{\bar u_t(x^w_t)}^2 + \big(w_t(x^w_t) - \bar u_t(x^w_t)\big) \cdot \bar g_t(x^w_t).
	\end{equation}
	Completing the square in $w_t(x^w_t)$,
	\begin{equation}
		\label{eqn:perf_proof_completed_square}
		\frac{d}{dt} V^{\bar u}_t(x^w_t) = \tfrac{1}{2} \norm{\bar u_t(x^w_t) - \bar g_t(x^w_t)}^2 - \tfrac{1}{2} \norm{w_t(x^w_t) - \bar g_t(x^w_t)}^2 + \tfrac{1}{2} \norm{w_t(x^w_t)}^2.
	\end{equation}
	Integrating from $s$ to $1$ along the trajectory and using the terminal condition $V^{\bar u}_1(x) = \lambda\, r(x)$,
	\begin{equation}
		\label{eqn:perf_proof_integrated}
		\lambda\, r(x^w_1) - V^{\bar u}_s(x) = \int_s^1 \!\left[\, \tfrac{1}{2} \norm{\bar u_t(x^w_t) - \bar g_t(x^w_t)}^2 - \tfrac{1}{2} \norm{w_t(x^w_t) - \bar g_t(x^w_t)}^2 + \tfrac{1}{2} \norm{w_t(x^w_t)}^2\, \right] dt.
	\end{equation}
	From the definition of the value function, $V^w_s(x) = \lambda\, r(x^w_1) - \int_s^1 \tfrac{1}{2}\norm{w_t(x^w_t)}^2\, dt$, so subtracting the integral of $\tfrac{1}{2}\norm{w_t(x^w_t)}^2$ from both sides of~\cref{eqn:perf_proof_integrated} yields the performance-difference identity~\cref{eqn:performance_diff} for any candidate control $w$.

	Specializing to $w = \bar g$ makes the second norm in~\cref{eqn:performance_diff} vanish, giving
	\begin{equation}
		\label{eqn:perf_proof_policy_improvement}
		V^{\bar g}_s(x) - V^{\bar u}_s(x) = \frac{1}{2} \int_s^1 \norm{\bar u_t(x^{\bar g}_t) - \bar g_t(x^{\bar g}_t)}^2\, dt \;\geq\; 0,
	\end{equation}
	with equality if and only if $\bar u_t(x^{\bar g}_t) = \bar g_t(x^{\bar g}_t)$ for almost every $t \in [s, 1]$.
	At equality, $\bar u = \bar g = \nabla_x V^{\bar u}$, and substituting this into the policy-evaluation identity~\cref{eqn:doc_policy_eval} of~\cref{lem:policy_evaluation} reduces it to
	\begin{equation}
		\label{eqn:perf_proof_hjb_match}
		\partial_t V^{\bar u}_t(x) + b_t(x) \cdot \nabla_x V^{\bar u}_t(x) + \tfrac{1}{2} \norm{\nabla_x V^{\bar u}_t(x)}^2 = 0,
		\qquad V^{\bar u}_1(x) = \lambda\, r(x),
	\end{equation}
	which is exactly the reduced HJB equation~\cref{eqn:doc_hjb_reduced} whose unique solution is the optimal value function $V^*$.
	Hence $V^{\bar u} = V^*$ and $\bar u = \nabla_x V^{\bar u} = \nabla_x V^* = u^*$ by the optimal control formula~\cref{eqn:doc_optimal_control}.
\end{proof}

\subsection{\texorpdfstring{Proof of \cref{prop:wtf_critical_point}}{Proof of WTF critical-point optimality}}
\label{app:wtf_critical_point}
\wtfcriticalpoint*
\begin{proof}
	The proof proceeds in three steps.
	Step~1 analyzes the off-diagonal distillation critical condition; with the stop-gradient target frozen the loss is a convex quadratic functional whose minimum is zero and attained, so every critical point forces $X^{\hat w}_{s, t}$ to be the flow map of the velocity $\hat w_{t, t}$ on the support of the sampling distribution over the upper triangle of the $(s, t)$ plane.
	Step~2 analyzes the diagonal regression critical condition; using Step~1's flow map property it establishes unbiasedness of the Monte Carlo value estimator, and forces $\hat w_{t, t} = b_t + \nabla_x V^{\bar w}_t$ on the on-policy support.
	Step~3 stitches these together via the stop-gradient consistency $\bar w = \sg{\hat w}$ and the policy-iteration fixed-point characterization of~\cref{prop:performance_diff}.

	For the proof, we decompose $\calL_{\WTF}(\hat w) = \calL_{\mathrm{diag}}(\hat w; \bar w) + \beta\, \calL_{\dist}(\hat w; \bar w)$ with $\bar{w} = \sg{\hat{w}}$.
	The variational analysis below uses this factorization throughout, treating every stop-gradient quantity as held fixed under variations of $\hat w$, including those depending on $\bar w$ and the distillation target $\sg{\nabla X^{\hat w}_{t, \tau}\, \hat w_{t, t}}$, which is a function of $\hat w$ itself.
	We work with the Eulerian self-distillation loss~\cref{eqn:dist_eulerian}, but the Lagrangian and progressive variants follow by analogous arguments using the corresponding flow map characterization in~\cref{eqn:app_three_characterizations}.
	The full-support hypothesis on $p_{s, t}$ is satisfied by the uniform-on-upper-triangle sampler used in our algorithm, and $\hat w$ varies over the class $\calW$ of~\cref{assump:flowmap_class}.
	Since $\rho_0$ has full support and $X^{\bar w}_{0, t}$ is a homeomorphism, the law of $\bar x_t$ has full support on $\R^d$ for every $t$; combined with the full support of $p_{s, t}$ this makes $\supp(\rho^{\bar w}_{t, \tau, x}) = \{0 \le t \le \tau \le 1\} \times \R^d$.

	\paragraph{Step 1: Off-diagonal distillation.}
	The Eulerian self-distillation loss~\cref{eqn:dist_eulerian} reads
	\begin{equation}
		\label{eqn:wtf_proof_dist}
		\calL_{\dist}(\hat w; \bar w) = \E_{x_0, t, \tau}\!\left[\,\norm{\partial_t X^{\hat w}_{t, \tau}(\bar x_t) + \sg{\,\nabla X^{\hat w}_{t, \tau}(\bar x_t)\, \hat w_{t, t}(\bar x_t)\,}}^2\,\right],
	\end{equation}
	with $(t, \tau)$ sampled from the off-diagonal sampler $p_{s, t}$ on the upper triangle and $\bar x_t = X^{\bar w}_{0, t}(x_0)$.
	Here and below $\partial_t$ differentiates the \emph{first} time argument of the flow map, as in the Eulerian characterization of~\cref{eqn:app_three_characterizations}.
	Write $\rho^{\bar w}_{t, \tau, x}$ for the joint density of $(t, \tau, \bar x_t)$ under the upper-triangle sampler and the on-policy trajectory, and
	\begin{equation}
		\label{eqn:wtf_proof_dist_target}
		c(t, \tau, x) \coloneqq \nabla X^{\hat w}_{t, \tau}(x)\, \hat w_{t, t}(x)
	\end{equation}
	for the stop-gradient target, which the semi-gradient convention holds fixed under variations of $\hat w$.
	The off-diagonal velocities $\hat w_{t, \tau}$, $t < \tau$, therefore enter~\cref{eqn:wtf_proof_dist} only through the un-stopped factor $\partial_t X^{\hat w}_{t, \tau}(\bar x_t)$.
	Under the flow map parameterization $X^{\hat w}_{t, \tau}(x) = x + (\tau - t)\, \hat w_{t, \tau}(x)$,
	\begin{equation}
		\label{eqn:wtf_proof_dt_affine}
		\partial_t X^{\hat w}_{t, \tau}(x) = -\hat w_{t, \tau}(x) + (\tau - t)\, \partial_t \hat w_{t, \tau}(x),
	\end{equation}
	which is \emph{linear} in $\hat w$.
	With $c$ frozen, the residual in~\cref{eqn:wtf_proof_dist} is thus affine in the off-diagonal velocities and
	\begin{equation}
		\label{eqn:wtf_proof_dist_convex}
		\hat w \longmapsto \calL_{\dist}(\hat w; \bar w) = \E_{x_0, t, \tau}\!\left[\, \norm{\partial_t X^{\hat w}_{t, \tau}(\bar x_t) + c(t, \tau, \bar x_t)}^2\, \right]
	\end{equation}
	is a convex quadratic functional of them.
	Every stationary point of a convex quadratic functional is a global minimizer, so it is enough to identify the global minimum.

	That minimum is zero, and it is attained within the parameterized class.
	The field
	\begin{equation}
		\label{eqn:wtf_proof_dist_realizer}
		\hat w^\star_{t, \tau}(x) = \frac{1}{\tau - t}\int_t^\tau c(\sigma, \tau, x)\, d\sigma, \qquad t < \tau,
	\end{equation}
	leaves the diagonal untouched -- since $X^{\hat w}_{\tau, \tau} = \mathrm{id}$ gives $\nabla X^{\hat w}_{\tau, \tau} = I$ and hence $c(\tau, \tau, x) = \hat w_{\tau, \tau}(x)$, so~\cref{eqn:wtf_proof_dist_realizer} extends continuously to $\hat w^\star_{\tau, \tau} = \hat w_{\tau, \tau}$ -- and lies in the same class of two-time velocity fields as $\hat w$, with $\hat w^\star$ inheriting the regularity of $c$.
	It satisfies $X^{\hat w^\star}_{t, \tau}(x) = x + \int_t^\tau c(\sigma, \tau, x)\, d\sigma$ and therefore $\partial_t X^{\hat w^\star}_{t, \tau}(x) = -c(t, \tau, x)$, so its residual vanishes identically and $\calL_{\dist}(\hat w^\star; \bar w) = 0$.

	This is where the argument differs from a variational-derivative calculation: rather than dividing out the Jacobian factor $\partial(\partial_t X^{\hat w}_{t, \tau})/\partial \hat w_{t, \tau}$ and arguing that it is non-degenerate, convexity of the frozen objective together with realizability of the frozen target gives the conclusion with nothing to invert.
	Since a critical point attains the global minimum $0$ of~\cref{eqn:wtf_proof_dist_convex}, the residual vanishes $\rho^{\bar w}_{t, \tau, x}$-almost everywhere:
	\begin{equation}
		\label{eqn:wtf_proof_dist_critical}
		\partial_t X^{\hat w}_{t, \tau}(x) = -\nabla X^{\hat w}_{t, \tau}(x)\, \hat w_{t, t}(x), \qquad (t, \tau, x) \in \supp\big(\rho^{\bar w}_{t, \tau, x}\big).
	\end{equation}
	Equation~\cref{eqn:wtf_proof_dist_critical} is the Eulerian transport equation~\cref{eqn:app_three_characterizations} characterizing $X^{\hat w}_{s, t}$ as the flow map generated by the velocity field $\hat w_{t, t}$.
	Because it holds for every $x \in \R^d$, it holds in particular along the integral curves of $\hat w_{t, t}$.
	Combined with the diagonal initial condition $X^{\hat w}_{t, t}(x) = x$ -- which holds for free under the flow map parameterization $X^{\hat w}_{s, t}(x) = x + (t - s)\, \hat w_{s, t}(x)$ -- standard uniqueness results for the method of characteristics imply
	\begin{equation}
		\label{eqn:wtf_proof_dist_fixed}
		\partial_t X^{\hat w}_{s, t}(x) = \hat w_{t, t}\big(X^{\hat w}_{s, t}(x)\big),
		\qquad
		X^{\hat w}_{s, t} = X^{\hat w}_{\sigma, t} \circ X^{\hat w}_{s, \sigma},
		\qquad 0 \le s \le \sigma \le t \le 1,
	\end{equation}
	that is, $X^{\hat w}$ is the flow map generated by its own diagonal velocity.
	In particular, $X^{\hat w}$ -- and hence $X^{\bar w}$ at the stop-gradient consistency $\bar w = \sg{\hat w}$ -- satisfies the semigroup property.
	We reserve the symbol $u$ for the residual control $\hat w_{t, t} - b_t$ introduced in Step~3, so that $X^u$ keeps its usual meaning as the flow of $b_t + u_t$.

	\paragraph{Step 2: Diagonal regression.}
	The diagonal regression term reads
	\begin{equation}
		\label{eqn:wtf_proof_diag}
		\calL_{\mathrm{diag}}(\hat w; \bar w) = \E_{x_0, t, \tau}\!\left[\,\norm{\hat w_{t, t}(\bar x_t) - b_t(\bar x_t) - \widehat g(t, \bar x_t; \tau)}^2\,\right],
	\end{equation}
	with $t \sim \Unif[0, 1]$, $\tau \sim \Unif[t, 1]$, $\bar x_t = X^{\bar w}_{0, t}(x_0)$, and $\widehat g(t, x; \tau) = \nabla_x \widehat V(t, x; \tau)$ the value-gradient estimator obtained by differentiating the Monte Carlo estimator~\cref{eqn:mc_value} at the auxiliary sample $\tau$.

	The unbiasedness of $\widehat V$ as an estimator of $V^{\bar w}_t(x)$ relies on the semigroup property of $X^{\bar w}$ established in Step~1.
	The construction of~\cref{eqn:mc_value} writes $x_1 = X^{\bar w}_{\tau, 1}(X^{\bar w}_{t, \tau}(x))$, and the semigroup property collapses this to $X^{\bar w}_{t, 1}(x)$ independently of $\tau$.
	Combined with the substitution $(1 - t)\, \E_\tau[\norm{\cdot}^2] = \int_t^1 \norm{\cdot}^2\, d\tau$ for the control-cost term, this yields
	\begin{equation}
		\label{eqn:wtf_proof_unbiased}
		\begin{aligned}
			\E_{\tau \sim \Unif[t, 1]}\!\left[\widehat V(t, x; \tau)\right]
			 & = \lambda\, r\!\left(X^{\bar w}_{t, 1}(x)\right)
			- \frac{1}{2} \int_t^1 \norm{\bar w_{\tau, \tau}(X^{\bar w}_{t, \tau}(x)) - b_\tau(X^{\bar w}_{t, \tau}(x))}^2\, d\tau \\
			 & = V^{\bar w}_t(x),
		\end{aligned}
	\end{equation}
	where the second equality is the definition~\cref{eqn:value_fn} of $V^{\bar w}_t$ with control $u_\tau = \bar w_{\tau, \tau} - b_\tau$.
	Differentiating in $x$ and interchanging differentiation with expectation under the standard regularity conditions yields
	\begin{equation}
		\label{eqn:wtf_proof_grad_unbiased}
		\E_\tau\!\left[\widehat g(t, x; \tau)\right] = \nabla_x V^{\bar w}_t(x).
	\end{equation}

	Now we compute the variational derivative of $\calL_{\mathrm{diag}}$ with respect to $\hat w_{t, t}$ at a test point $x$.
	Since the regression target $b_t(\bar x_t) + \widehat g(t, \bar x_t; \tau)$ is a function of $\bar w$ only, it is held fixed under variations of $\hat w_{t, t}$.
	With its target frozen, the distillation term depends on $\hat w$ only through its off-diagonal components, so it contributes nothing to a diagonal variation and the critical-point condition for $\calL_{\WTF}$ reduces to that for $\calL_{\mathrm{diag}}$.
	The variational derivative is
	\begin{equation}
		\label{eqn:wtf_proof_diag_var}
		\frac{\delta \calL_{\mathrm{diag}}}{\delta \hat w_{t, t}(x)}
		= 2\, \E_\tau\!\left[\hat w_{t, t}(x) - b_t(x) - \widehat g(t, x; \tau)\right] \rho^{\bar w}(t, x),
	\end{equation}
	where $\rho^{\bar w}(t, x)$ denotes the joint density of $(t, \bar x_t)$ under $t \sim \Unif[0, 1]$, $x_0 \sim \rho_0$.
	By~\cref{eqn:wtf_proof_grad_unbiased}, the bracketed expectation simplifies to $\hat w_{t, t}(x) - b_t(x) - \nabla_x V^{\bar w}_t(x)$.

	Setting the variational derivative to zero on the support of $\rho^{\bar w}$ yields
	\begin{equation}
		\label{eqn:wtf_proof_diag_critical}
		\hat w_{t, t}(x) = b_t(x) + \nabla_x V^{\bar w}_t(x), \qquad (t, x) \in \supp(\rho^{\bar w}).
	\end{equation}
	Because $\bar w = \sg{\hat w}$, the right-hand side is the value gradient under $\hat w$ itself, and the on-policy support coincides with that of $\rho^{\hat w}$:
	\begin{equation}
		\label{eqn:wtf_proof_diag_fixed}
		\hat w_{t, t}(x) = b_t(x) + \nabla_x V^{\hat w}_t(x), \qquad (t, x) \in \supp(\rho^{\hat w}).
	\end{equation}

	\paragraph{Step 3: Joint consistency closes the loop.}
	From \cref{eqn:wtf_proof_diag_fixed},
	\begin{equation}
		\label{eqn:wtf_proof_diag_summary}
		\hat w_{t, t}(x) = b_t(x) + \nabla_x V^{\hat w}_t(x) \qquad \text{on } \supp(\rho^{\hat w}).
	\end{equation}
	Here $\rho^{\hat w}$ denotes the joint law of $(t, X^{\hat w}_{0, t}(x_0))$ under $t \sim \Unif[0, 1]$ and $x_0 \sim \rho_0$, which is the law $\rho^{\bar w}$ of~\cref{eqn:wtf_proof_diag_var} evaluated at the stop-gradient consistency $\bar w = \sg{\hat w}$.
	From \cref{eqn:wtf_proof_dist_fixed}, $X^{\hat w}_{s, t}$ is the flow map generated by the velocity $\hat w_{t, t}$.
	Define the residual control $u_t \coloneqq \hat w_{t, t} - b_t$, so that $\hat w_{t, t} = b_t + u_t$ is the controlled drift and $X^{\hat w}_{s, t} = X^u_{s, t}$; accordingly we write $\rho^u = \rho^{\hat w}$ for the on-policy law.

	Substituting $\hat w_{t, t} = b_t + u_t$ into~\cref{eqn:wtf_proof_diag_summary} gives the on-policy fixed-point condition
	\begin{equation}
		\label{eqn:wtf_proof_fixed_point}
		u_t(x) = \nabla_x V^{u}_t(x), \qquad (t, x) \in \supp(\rho^{u}),
	\end{equation}
	where $V^u_t \equiv V^{\hat w}_t$ since the value function depends only on the controlled drift $b_t + u_t$.

	Since $\rho_0$ has full support on $\R^d$ and the controlled flow $X^u_{0, t}$ is a diffeomorphism under standard regularity conditions, the on-policy density $\rho^u$ has full support.
  Hence~\cref{eqn:wtf_proof_fixed_point} gives $u_t = \nabla_x V^u_t$ everywhere, and the equality clause of~\cref{prop:performance_diff} (with $\bar u = u$ and $\bar g = \nabla_x V^u_t$) yields $u = u^*$, the optimal control of~\cref{eqn:oc_problem}.

	Combining the above arguments, we conclude that $\hat w_{t, t} = b_t + u^*_t$ and that $X^{\hat w}_{s, t}$ is the flow map of $b_t + u^*_t$.
\end{proof}

\section{Algorithmic aspects}
\label{sec:algorithmic}
We next describe the implementation of the WTF objective~\cref{eqn:val_objective}, summarized in~\cref{alg:wtf}.

\paragraph{Euclidean reward gradient.}
The exact reward contribution to the value gradient is $\nabla X_{t, 1}(x_t)^\top \nabla r(x_1)$.
For our image experiments, we instead use $\nabla r(x_1)$ while leaving the endpoint $x_1 = X_{t, 1}(x_t)$ unchanged.
This removes the flow map Jacobian from the reward gradient and computes the update in the direction that increases reward at the generated endpoint.
Following the terminology of \citet{huang2026guideflowfewstepalignment}, we refer to this as the \emph{Euclidean} reward gradient.
For the residual parameterization $X_{t, 1}(x) = x + (1 - t)\, w_{t, 1}(x)$, we implement it as
\begin{equation*}
	x_1 = x_t + \sg{X_{t, 1}(x_t) - x_t} .
\end{equation*}
This leaves the endpoint unchanged while omitting the flow map Jacobian from the backward pass.
Related updates have been used for diffusion reward fine-tuning, including DRTune~\citep{wu2024deepreward}.
\Cref{alg:wtf} shows both gradient choices, and \cref{ssec:exp_detach} compares them empirically.

\paragraph{Parameterization.}
The derivation writes the control as a residual $u$ on the base drift $b$.
We fine-tune a single flow map
\begin{equation}
	\label{eqn:w-flow-map}
	X^{\hat w}_{s, t}(x) = x + (t - s)\, \hat w_{s, t}(x),
\end{equation}
The network is initialized from the pre-trained mean velocity $v_{s, t}$.
The residual is recovered as $u_{s, t} = \hat{w}_{s, t} - v_{s, t}$ and is used only in the objective.
\Cref{app:diffusion_convention} gives the corresponding diffusion-time convention.

\paragraph{Off-diagonal self-distillation.}
The off-diagonal regularizer $\calL_{\dist}$ in~\cref{eqn:val_objective} can be any of the standard self-distillation losses for flow maps; we use the Eulerian (mean-flow) variant~\citep{boffi_flow_2024,boffi2025self,sabour2025align,geng2025meanflowsonestepgenerative}
\begin{equation}
	\label{eqn:dist_eulerian}
	\calL_{\dist}(\hat w) = \E\!\left[\norm{\partial_t X^{\hat w}_{t, \tau}(x_t) + \sg{\nabla X^{\hat w}_{t, \tau}(x_t)\, \hat w_{t, t}(x_t)}}^2\right],
\end{equation}
where $x_t = \sg{X^{\bar w}_{0, t}(x_0)}$ is sampled on-policy along the current flow and $\sg{\cdot}$ denotes a stop-gradient.
We use the Eulerian self-distillation objective, matching the objective used to pre-train the base models.

\paragraph{Reward-gradient scaling.}
Because the base velocity and reward gradient can have very different numerical scales, we normalize the reward gradient by their norm ratio and use $\lambda_{\mathrm{eff}} = \lambda \kappa$ in~\cref{eqn:val_objective}.
\Cref{app:reward_scale} gives the definition, and \cref{tab:impl_t2i} reports the raw values of $\lambda$.

\subsection{Diffusion-time convention}
\label{app:diffusion_convention}
Many pre-trained checkpoints, including score-based, variance-preserving, and EDM-style models~\citep{karras_elucidating_2022}, use time $t = 1$ for the Gaussian base distribution and $t = 0$ for the data distribution.
The reward is evaluated at $t = 0$, and the signs and integration limits differ from the convention in the main text.
We state the resulting objective and value-gradient target below.

Let $b_t$ denote a pre-trained velocity for which integrating $\dot{x}_t = b_t(x_t)$ from $t=1$ to $t=0$ transports a Gaussian sample $x_1 \sim \Normal(0, \Id)$ into a data sample $x_0 \sim \rho^*$, and let $X_{s,t}^u(x)$ denote the state at time $t$ under the controlled ODE $\dot{x}_\tau^u = b_\tau(x_\tau^u) + u_\tau(x_\tau^u)$ started from $x$ at time $s$.
In this convention, one typically has $t \leq s$.
The OC problem~\cref{eqn:oc_problem} becomes
\begin{equation}
	\label{eqn:oc_problem_diff}
	\begin{aligned}
		 & \sup_u \E_{x_1 \sim \Normal(0, \Id)}\!\left[\, \lambda\, r(x_0^u) - \frac{1}{2} \int_0^1 \norm{u_t(x_t^u)}^2\, dt\, \right], \\
		 & \text{subject to}\quad \dot{x}_t^u = b_t(x_t^u) + u_t(x_t^u), \quad x_1^u = x_1.
	\end{aligned}
\end{equation}
Only the terminal-time indexing changes relative to~\cref{eqn:oc_problem}, since the control-cost integrand is orientation-free.
We note that the integral of the control runs opposite to the direction of integration of the flow, which ensures the control cost remains positive.

The value function under control $u$ is
\begin{equation}
	\label{eqn:value_fn_diff}
	V^u_t(x) := \lambda\, r\!\left(X_{t,0}^u(x)\right) - \frac{1}{2} \int_0^t \norm{u_\tau(X_{t,\tau}^u(x))}_2^2\,d\tau,
\end{equation}
where again the integration domain is the unsigned physical interval $[0, t]$ so that the control cost is accumulated as a positive quantity.
Applying dynamic programming with infinitesimal step $t \mapsto t-\epsilon$ toward the data gives the Hamilton--Jacobi--Bellman equation
\begin{equation}
	\label{eqn:hjb_diff}
	\p_t V^*_t(x) + b_t(x) \cdot \nabla_x V^*_t(x) - \frac{1}{2}\norm{\nabla_x V^*_t(x)}_2^2 = 0, \qquad V^*_0(x) = \lambda\, r(x),
\end{equation}
with the optimal control
\begin{equation}
	\label{eqn:optimal_control_diff}
	u_t^*(x) = -\nabla_x V^*_t(x).
\end{equation}
The minus sign in~\cref{eqn:optimal_control_diff}, in contrast to $u^*_t = +\nabla_x V^*_t$ in the data-terminal convention, is the consequence of time reversal.
Because integration of the generative process flows backwards in time, this negative sign implements gradient \textit{ascent} at inference, matching the forward-time result.

The single-sample Monte Carlo estimator of~\cref{eqn:value_fn_diff} also changes accordingly.
Fix a frozen reference $\bar{u}$, sample $\tau \sim \Unif[0, t]$, and compute $x_\tau = X_{t,\tau}^{\bar{u}}(x)$ and $x_0 = X_{\tau, 0}^{\bar{u}}(x_\tau)$ to form
\begin{equation}
	\label{eqn:mc_value_diff}
	\widehat{V}_t(x) = \lambda\, r(x_0) - \frac{t}{2}\norm{\bar{u}_\tau(x_\tau)}_2^2,
\end{equation}
where the factor $t$ is the length of the remaining interval $[0, t]$, replacing the factor $1-t$ from~\cref{eqn:mc_value}.
Differentiation yields the gradient target $\widehat{g}_t(x) = \nabla_x \widehat{V}_t(x)$, and the diagonal value gradient regression of~\cref{eqn:val_objective} fits $\hat w_{t, t}(\bar{x}_t) \to b_t(\bar{x}_t) - \widehat{g}_t(\bar{x}_t)$ along the reference trajectory $\bar{x}_t = X_{1, t}^{\bar{u}}(x_1)$, with the sign flip on $\widehat{g}_t$ matching~\cref{eqn:optimal_control_diff}.

\subsection{Reward-gradient scaling}
\label{app:reward_scale}
The reward gradient and the base flow velocity gradient typically live on different numerical scales.
The base flow $b_t$ has been trained to push $\rho_0$ all the way to $\rho_1$ in unit time, so the magnitude of $b_t$ is set by the dataset and the training schedule.
Hand-coded or learned rewards typically have gradient magnitudes set by an unrelated convention.
Without rescaling, $\lambda$ in the value gradient~\cref{eqn:mc_value} must be tuned by orders of magnitude per reward to balance $\lambda \nabla r$ against the implicit scale of the base velocity field.
To normalize this in standardized units, we rescale the reward gradient by the ratio of the base flow velocity norm to the reward gradient norm:
\begin{equation}
	\label{eqn:tcf}
	\kappa = \E\left[\frac{\norm{b_t(\bar{x}_t)}_2}{\norm{\nabla_x r(\bar{x}_1)}_2}\right], \qquad \lambda_{\mathrm{eff}} = \lambda \cdot \kappa,
\end{equation}
and use $\lambda_{\mathrm{eff}}$ in place of $\lambda$ inside~\cref{eqn:mc_value}.
We track the two norms with exponential moving averages and do not freeze them after warmup.
This requires one reward backward pass and one evaluation of $b_t$ per iteration.
\Cref{ssec:exp_lambda} sweeps $\lambda$ on text-to-image.

\section{Implementation details}
\label{app:hyper}

\subsection{Training setup}
\label{app:hyper_t2i}
\paragraph{ImageNet.}
We fine-tune DMF XL/2~\citep{lee2025decoupled} end to end using HPSv2 as the reward, starting from the pre-trained \texttt{dmf\_xl\_2\_256} checkpoint at $256 \times 256$ resolution over the full $1000$-class conditioning.
We update the full network without an adapter.
A single parameterization is used for $\hat{w}_{t, t}$ and $\hat{w}_{s, t}$.
Training uses \texttt{bf16} mixed precision for $12{,}500$ optimizer steps, with the EMA weights as the frozen reference $\bar w$ in~\cref{eqn:val_objective}.
The text-to-image experiments instead use a stop-gradient copy of the current weights.
\paragraph{Text-to-image.}
We fine-tune the pre-trained TiM-T2I model~\citep{wang2025transition} at $512 \times 512$ resolution using the \texttt{tim\_xl\_p1\_t2i} configuration.
Images are represented on a $16 \times 16$ latent grid with $32$ channels from a frozen deep compression autoencoder~\citep{chen2024deep} (\texttt{mit-han-lab/dc-ae-f32c32-sana-1.1-diffusers}), which downsamples by $32$ spatially to $32$ latent channels.
Captions are encoded by a frozen Gemma~3 1B instruction-tuned text encoder~\citep{gemma2025three} (\texttt{google/gemma-3-1b-it}) with a maximum sequence length of $256$.
We train LoRA adapters on all attention $(Q, K, V, \text{out})$ and MLP $(\text{fc}_1, \text{fc}_2)$ projections in every transformer block.
Prompts are drawn from the photo and painting categories of HPDv2.
\paragraph{Fine-tuning.}
A single LoRA adapter of rank $16$ with scaling $\alpha = 16$ is used for the diagonal update and all off-diagonal flow map evaluations.
We train for $1{,}460$ optimizer steps and use the Euclidean reward gradient described in~\cref{sec:algorithmic}.
\paragraph{Reward model.}
Reward gradients use HPSv2 v2.1 with the fine-tuned \texttt{xswu/HPSv2} weights on an OpenCLIP~\citep{ilharco2021openclip} CLIP ViT-H-14 backbone.
Images in $[0, 1]$ are processed at $224$ pixels with the HPSv2 mask-aware normalization and resize pipeline.
\paragraph{Optimization.}
Both benchmarks use AdamW. ImageNet uses $(\beta_1, \beta_2) = (0.9, 0.95)$ with weight decay $0$, so its update coincides with Adam; text-to-image uses $(\beta_1, \beta_2) = (0.9, 0.999)$ with weight decay $10^{-2}$.
The learning rate is constant.
\Cref{tab:impl_t2i} quotes the global batch, together with the per-GPU count and accumulation factor that produce it across the $8$ GPUs.
The EMA decay $\mu$ averages the trainable parameters of the fine-tuned model, updated once per optimizer step.
For text-to-image, we clip each per-sample reward gradient to the $0.8$ quantile of the batch gradient norms.
For both benchmarks, we apply a global parameter-gradient $L_2$ clip of $1$.
The reward scale $\lambda$ in~\cref{tab:impl_t2i} is the raw coefficient before the normalization in~\cref{app:reward_scale}.
Every sample contributes to both the diagonal value-gradient term and the regularizer; $p_{\mathrm{diag}}$ mixes the times at which the regularizer is applied rather than routing samples between losses.
The text-to-image runs use a different regularizer and carry no such mixing probability.
\Cref{tab:impl_t2i} gives the configuration for both benchmarks.

\subsection{Sampling and evaluation}
\label{app:hyper_eval}
\paragraph{ImageNet.}
We evaluate at NFE $\in \{1, 250\}$ with CFG scale $1.0$.
Using the first $32$ ImageNet classes, we generate $16$ samples per class, conditioned on the class label.
We use a fixed evaluation seed and share initial latents across methods.
HPSv2, PickScore, and ImageReward are averaged over the same $512$ images.
Diversity is the class-averaged mean pairwise squared distance in DreamSim and CLIP embedding space.
WTF results are means over three matched training seeds.
We use the first $32$ ImageNet classes, fixed across all methods.
\paragraph{Text-to-image.}
We evaluate at NFE $\in \{1, 2, 4, 8, 50\}$ with CFG scale $2.5$ and share starting latents across methods.
The evaluation pool contains $100$ fixed prompts from the photo and painting categories of HPDv2.
Reward metrics use $1{,}000$ scored generations drawn from this pool under a fixed evaluation seed, so every method is scored on the identical prompt sequence.
Diversity is computed from $16$ samples for each of a fixed subset of $32$ prompts using the same DreamSim and CLIP pairwise distances.
Each sampler transition uses one network evaluation.
A joint conditional and unconditional forward pass counts as one NFE.
ImageNet uses CFG scale $1.0$ without an unconditional branch.

\begin{table}[t]
	\centering
	\small
	\begin{tabular}{lcc}
		\toprule
		Setting                                  & ImageNet                            & Text-to-image                        \\
		\midrule
		Base flow map                            & DMF XL/2, $256^2$                   & TiM-T2I XL, patch $1$, $512^2$           \\
		Pre-trained checkpoint                   & \texttt{dmf\_xl\_2\_256}            & \texttt{TiM-T2I}                    \\
		Fine-tuned parameters                    & full network                        & one LoRA adapter, rank $16$         \\
		Optimizer                                & AdamW $(0.9, 0.95)$                 & AdamW $(0.9, 0.999)$                 \\
		Weight decay                             & $0$                                 & $10^{-2}$                            \\
		Learning rate                            & $1 \times 10^{-4}$                  & $1 \times 10^{-3}$                   \\
		Global batch size                        & $64$                                & $80$                                 \\
		\quad per GPU $\times$ accumulation      & $8 \times 1$                        & $2 \times 5$                         \\
		EMA decay $\mu$                          & $0.999$                             & $0.9999$                             \\
		Distillation weight $\beta$              & ---                                 & $0.5$                                \\
		Reward scale $\lambda$ (raw)             & $15$                                & $30$                                 \\
		Reward-gradient smoothing $(P, \sigma)$  & ---                                 & $(1,\ 0.20)$                         \\
		Reward-gradient quantile clip            & ---                                 & $0.8$                                \\
		Parameter-gradient L2 clip               & $1$                                 & $1$                                  \\
		Optimizer steps                          & $12{,}500$                          & $1{,}460$                                \\
		GPU configuration                        & $8 \times \mathrm{H100}$            & $8 \times \mathrm{H100}$             \\
		\bottomrule
	\end{tabular}
	\vspace{0.5em}
	\caption{\label{tab:impl_t2i}\label{tab:impl_imagenet}\textbf{Implementation details for both benchmarks.}
		Reward scales are the raw coefficient $\lambda$; The applied weight is $\lambda_{\mathrm{eff}}$ from~\cref{app:reward_scale}.}
\end{table}

\subsection{Baseline implementations}
\label{app:hyper_baselines}
Both baselines use the sampling and evaluation protocol of~\cref{app:hyper_eval}. Training stops
when the evaluation reward does not improve for six consecutive checkpoints.

Flow-GRPO uses group size $24$ with one prompt per rank, $192$ samples per iteration, $10$ SDE
steps, LoRA rank $32$ with $\alpha = 32$, learning rate $10^{-4}$, and weight decay $10^{-2}$.
Advantages use a per-prompt mean and a globally gathered reward standard deviation, clipped to
$[-5, 5]$. The KL penalty is the analytic same-variance Gaussian transition KL against the
frozen backbone, with weight $0.01$. Each rollout batch is followed by one optimizer update over all $10$ timesteps, and applies no EMA to the adapter weights.

Adjoint Matching~\citep{domingoenrich2025adjointmatchingfinetuningflow} uses reward scale
$1.2 \times 10^{5}$, $N = 40$ rollout steps, $K = 20$ timesteps per update, LoRA rank $8$,
learning rate $2 \times 10^{-5}$, and effective batch $28$. Half of the $K$ timesteps are drawn
uniformly from the first three quarters of the trajectory; the remaining $10$ are the final $10$
timesteps. Adapter parameters are float32. Loss terms whose norm exceeds an exponential moving
average of the globally gathered $0.9$ quantile are masked.

\subsection{Training cost and compute comparison}
\label{app:compute}
We compare fine-tuning compute on the same $8 \times \mathrm{H100}$ node.
\Cref{tab:compute} reports GPU-hours to each displayed checkpoint and excludes pre-training, evaluation, and post-hoc flow map distillation.
The text-to-image totals are approximate and use the median time between consecutive training logs.
\Cref{tab:compute} reports the cost to each displayed checkpoint, whereas \cref{fig:compute_efficiency} reports the compute WTF requires to reach each baseline's peak reward.
These comparisons use different endpoints, so their ratios are not directly commensurable.

\begin{table}[t]
	\centering
	\footnotesize
	\setlength{\tabcolsep}{5pt}
	\resizebox{\textwidth}{!}{%
		\begin{tabular}{@{}llll@{}}
			\toprule
			Method                  & Task     & GPU-hours & Output of fine-tuning \\
			\midrule
			Adjoint Matching        & ImageNet & $40$ & flow; few-step needs separate distillation \\
			MFM                     & ImageNet & $40$ & flow; few-step needs separate distillation \\
			VFM                     & ImageNet & $40$ & one-step sampler, $1$ NFE only \\
			\ourrow WTF (Ours)      & ImageNet & $22$ & flow map, any NFE \\
			\midrule
			Adjoint Matching        & T2I      & $48$ & flow; few-step needs separate distillation \\
			Flow-GRPO                & T2I      & $48$ & flow; few-step needs separate distillation \\
			\ourrow WTF (Ours)      & T2I      & $48$ & flow map, any NFE \\
			\bottomrule
		\end{tabular}%
	}
	\vspace{0.5em}
	\caption{\textbf{Fine-tuning compute.}
		GPU-hours to the reported checkpoint on the same $8 \times \mathrm{H100}$ node.
		Post-hoc flow map distillation is excluded.}
	\label{tab:compute}
\end{table}

\section{Additional experimental results}
\label{app:extras}

\subsection{Qualitative comparison on ImageNet}
\label{ssec:exp_imagenet_qual}
\Cref{fig:overview_imagenet} compares samples against the base model and adjoint matching at a fixed
initial latent.

\begin{figure}[t]
    \centering
    \includegraphics[width=\linewidth]{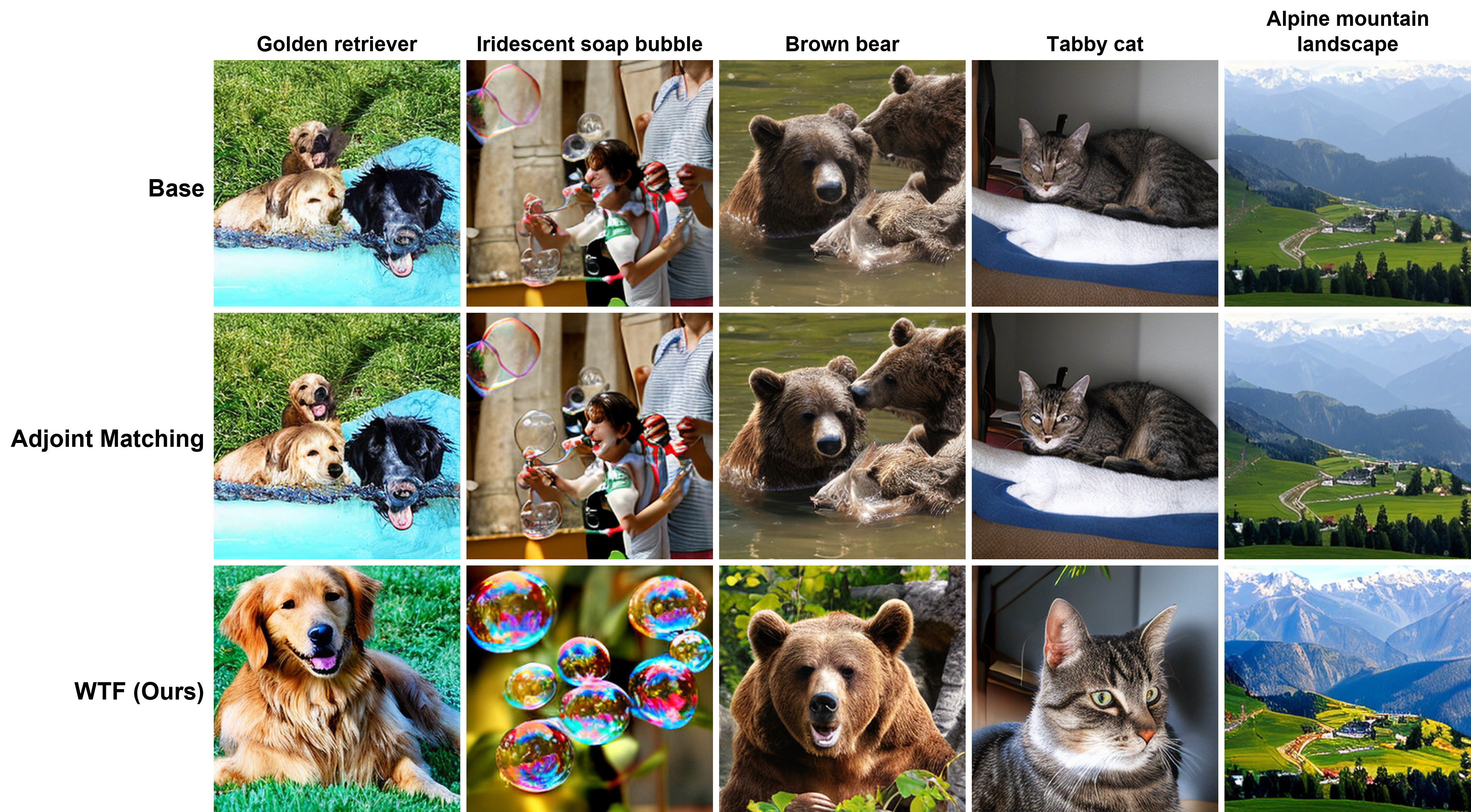}
    \caption{\textbf{Qualitative results on ImageNet.} Samples at 250 NFE from the base model, adjoint matching, and WTF, with the initial latent fixed down each column. Adjoint matching stays close to the base sample. WTF changes composition and color while keeping the class.}
    \label{fig:overview_imagenet}
\end{figure}

\subsection{Exact versus Euclidean reward gradient}
\label{ssec:exp_detach}
The exact reward contribution uses the flow map pullback $J^\top_{X_{t, 1}} \nabla r$, while the Euclidean estimator replaces it with $\nabla r$.
\Cref{tab:detach_ab} compares the two estimators on text-to-image at matched iterations.
The Euclidean estimator attains higher HPSv2 and higher DreamSim and CLIP diversity.

\begin{table}[t]
	\centering
	\begin{tabular}{lccccc}
		\toprule
		Value gradient                & HPSv2 $\uparrow$ & PickScore $\uparrow$ & ImageReward $\uparrow$ & DreamSim $\uparrow$ & CLIP diversity $\uparrow$ \\
		\midrule
		Exact, $J^\top \nabla r$      & $0.377$ & $22.59$ & $1.34$ & $0.163$ & $0.161$ \\
		\ourrow Euclidean, $\nabla r$ & $0.395$ & $22.83$ & $1.40$ & $0.195$ & $0.193$ \\
		\bottomrule
	\end{tabular}
	\vspace{0.5em}
	\caption{\textbf{Exact versus Euclidean reward gradients on text-to-image.}
		Both rows are evaluated at $800$ optimizer steps under one protocol
		($50$ NFE, guidance $2.5$, $1{,}000$ quality prompts, $32 \times 16$ diversity prompts).
		A matched comparison over three seeds at $600$ iterations and $\lambda = 25$ moves HPSv2
		from $0.359$ to $0.376$ and DreamSim from $0.154$ to $0.217$, in the same direction and of
		comparable size.}
	\label{tab:detach_ab}
\end{table}

\subsection{Reward scale and reward-diversity tradeoff}
\label{ssec:exp_lambda}
We retrain the text-to-image model over a $4\times$ range of reward scales at matched training steps and evaluation protocol.
Increasing $\lambda$ increases HPSv2 while decreasing DreamSim and CLIP diversity.
Increasing $\lambda$ from $10$ to $40$ raises HPSv2 by $0.023$ and lowers DreamSim and CLIP diversity by $0.022$ and $0.013$, respectively.

\section{Synthetic experiments: transport versus reweighting}
\label{app:synthetic}

We compare the WTF and KL population optima on settings where both laws can be computed exactly.
The experiments isolate how transport and reward reweighting produce different terminal laws.

\subsection{Gaussian benchmark}
\label{app:synthetic_lowdensity}
The base is $\rho_0 = \rho_1 = \Normal(0, 1)$ in one dimension, and the reward is a bounded bump $r_K(y) = \exp(-\tfrac{1}{2}((y-K)/w)^2)$ of width $w = 2.25$ centered at $K$ base standard deviations from the mean.
The width is chosen so that the reward overlaps the base appreciably rather than sitting in its tail: at $K = 4.75$ the base already attains $\E_{\rho_1}[r] = 0.142$, and $15.8\%$ of the base mass has $r > 0.25$.

Both laws are evaluated exactly rather than sampled.
We compute the KL tilt by numerical quadrature.
For the affine base drift, the prior-action cost has a closed form, so the WTF optimum is obtained from the pointwise maximization in~\cref{prop:pointwise-decomp} and its terminal density by change of variables.
\Cref{fig:synth_lowdensity} and~\cref{tab:synth_lowdensity} report the resulting population quantities.
\Cref{fig:synth_sweeps} quantifies the same comparison across the reward centre $K$ and the reward scale $\lambda$.

\begin{figure}[t]
    \centering
    \includegraphics[width=0.48\linewidth]{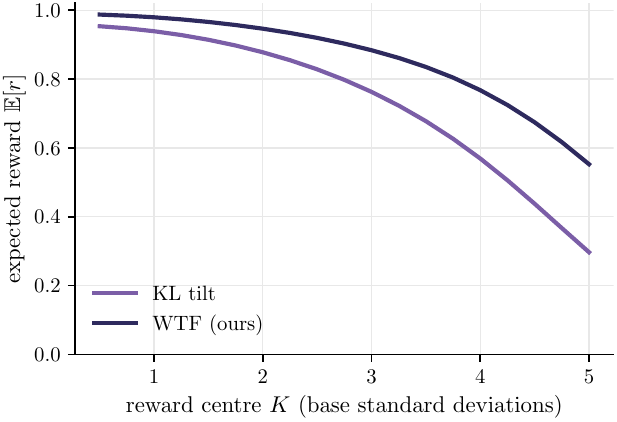}
    \hfill
    \includegraphics[width=0.48\linewidth]{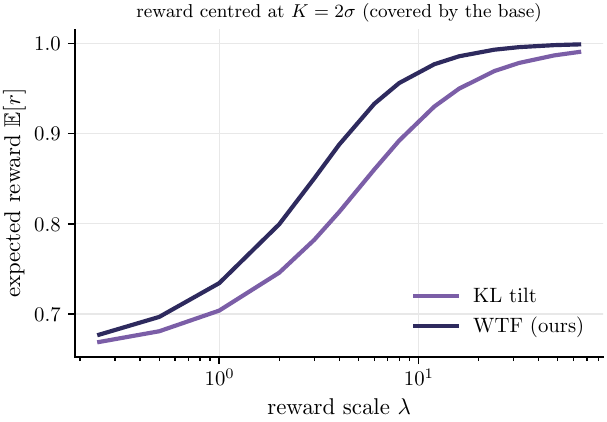}
\caption{\textbf{Quantitative sweeps for the one-dimensional comparison.}
Left: expected reward against the reward centre $K$ at fixed reward scale.
Right: expected reward against the reward scale $\lambda$ at fixed reward centre.
Both quantify the trend that~\cref{fig:synth_lowdensity} shows qualitatively.}
    \label{fig:synth_sweeps}
\end{figure}

At $\lambda = 7$ the difference between the two laws widens monotonically as the reward moves outward, from $0.068$ at $K = 2$ to $0.256$ at $K = 5$.

\begin{table}[t]
    \centering
    \begin{tabular}{lccccc}
        \toprule
        & \multicolumn{5}{c}{Expected reward $\E[r]$ at $\lambda = 7$, $w = 2.25$} \\
        \cmidrule(lr){2-6}
        Method & $K=2$ & $K=3$ & $K=4$ & $K=4.75$ & $K=5$ \\
        \midrule
        Base $\rho_1$      & 0.657 & 0.435 & 0.244 & 0.142 & 0.116 \\
        KL tilt            & 0.878 & 0.763 & 0.569 & 0.367 & 0.297 \\
        \ourrow WTF (Ours) & 0.946 & 0.884 & 0.768 & 0.617 & 0.553 \\
        \bottomrule
    \end{tabular}
    \vspace{0.5em}
    \caption{\textbf{Expected reward as the reward moves outward.} The reward bump of width $w = 2.25$ is centered $K$ base standard deviations from the mean. Both fine-tuned laws are population optima, the tilt by quadrature and WTF in closed form, so no seed variation enters.}
    \label{tab:synth_lowdensity}
\end{table}

\subsection{Qualitative comparison on text-to-image}
\label{app:qual_t2i}
The figures below show text-to-image samples across inference budgets, with the prompts listed in each caption.

\paragraph{Prompts and classes in \cref{fig:front_hero}.}
The text-to-image rows of \cref{fig:front_hero} use, from top to bottom,
``The image depicts the god dreaming at the end of time.'',
``An otherworldly world depicted with vivid colors by Fuco Ueda.'', and
``A skull-shaped island with rocks and vegetation, painted by Ghibli with strong light and shadow.''
The ImageNet rows are class-conditional rather than prompted, and show the classes
\emph{platypus}, \emph{Shetland sheepdog} and \emph{volcano}.
\begin{figure}[p]
	\centering
	\includegraphics[width=\linewidth]{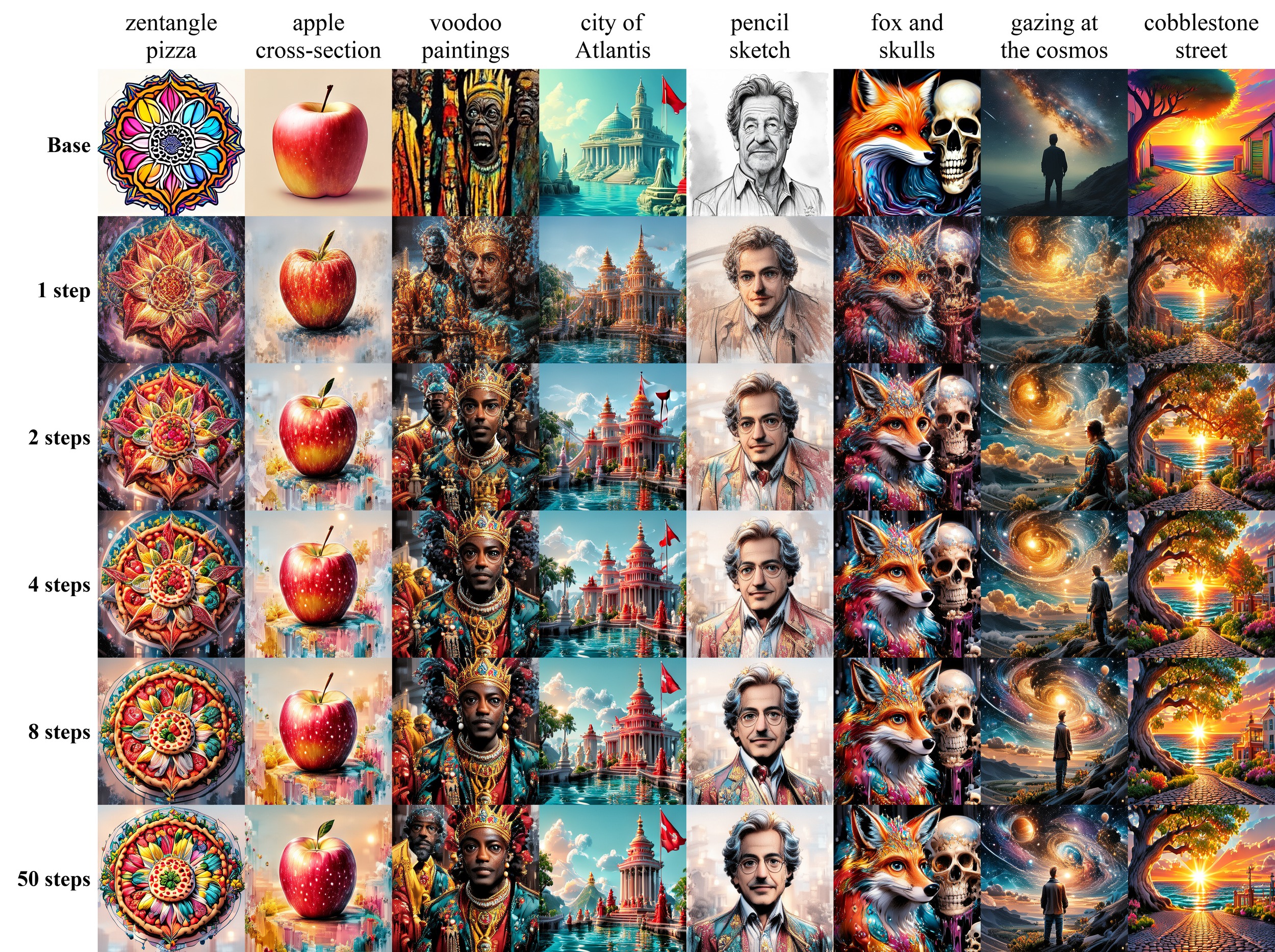}
	\caption{\textbf{Reward-aligned text-to-image samples across inference budgets (first set).}
		Each column is one prompt and each row an inference budget, with the pre-trained model in the top row.
		Within a column every WTF row uses the same initial noise, so the sequence shows one sample refining as the budget grows rather than independent draws.
		Prompts, left to right:
		(a) ``A zentangle pizza illustration with colorful ink.
		(b) ``Cross section of an apple in a limited neutral palette with a beautiful graphic design and a painterly style.
		(c) ``Scary African voodoo paintings by Jean-Michel Basquiat.
		(d) ``A digital painting of the legendary water city of Atlantis, featuring a Greek temple, statues, and a red flag.''
		(e) ``A pencil sketch of Danny Devito by Milt Kahl.
		(f) ``The image features a surreal fox and skulls in highly detailed, liquid oilpaint style.
		(g) ``An art piece by Wojciech Siudmak depicting an individual gazing at the vast cosmos.
		(h) ``A cobblestone street with a tree over the sea at sunset, illuminated by sun rays.
		}
	\label{fig:appx_t2i_1}
\end{figure}

\begin{figure}[p]
	\centering
	\includegraphics[width=\linewidth]{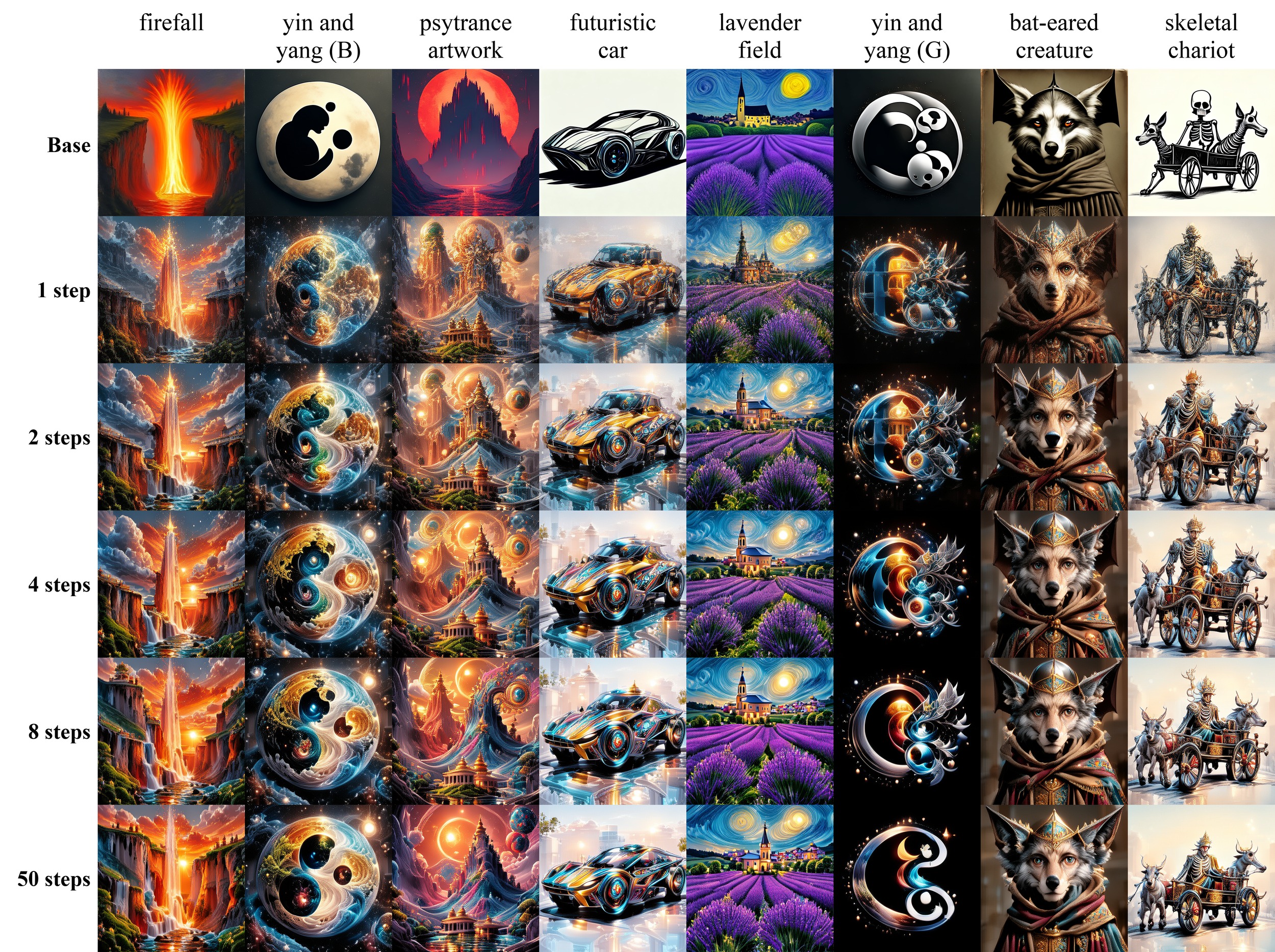}
	\caption{\textbf{Reward-aligned text-to-image samples across inference budgets (second set).}
		Each column is one prompt and each row an inference budget, with the pre-trained model in the top row.
		Within a column every WTF row uses the same initial noise, so the sequence shows one sample refining as the budget grows rather than independent draws.
		Prompts, left to right:
		(a) ``A painting of a firefall cascading over a high cliff.
		(b) ``An image depicting the concept of yin and yang.
		(c) ``Psytrance artwork by Lee Madgwick.
		(d) ``Artwork depicting a futuristic car, created by Ed Roth.
		(e) ``A night scene of a lavender field with a town and church in the background, reminiscent of Van Gogh.
		(f) ``An image depicting the concept of yin and yang.
		(g) ``Portrait of a creature with bat ears, a wolf snout and eagle features, wearing a poncho and helmet.
		(h) ``The image is a drawing of a skeletal, frail figure driving a chariot pulled by two skeletal hounds.
		}
	\label{fig:appx_t2i_2}
\end{figure}

\end{document}